%% file: main.tex
\documentclass[preprint,10pt,3p]{elsarticle}
\usepackage{geometry}
\usepackage{graphics} 
\usepackage{epsfig} 
\usepackage{times} 
\usepackage{amsmath} 
\usepackage{amssymb}  
\usepackage{epstopdf}
\usepackage[table]{xcolor} 
\usepackage{caption}
\usepackage{verbatim}
\usepackage{cases}
\allowdisplaybreaks[4]
\usepackage{enumerate}
\usepackage{booktabs}
\usepackage{balance}
\usepackage{mathbbol}
\usepackage{multirow}

\usepackage{algorithm}
\usepackage{algorithmicx}
\usepackage{algpseudocode}
\usepackage{mathrsfs}
\usepackage{threeparttable}
\usepackage[draft]{todonotes}
\usepackage{makecell}

\usepackage{enumitem}

\usepackage{booktabs} 

\usepackage{multirow}

\usepackage[colorlinks=false,      
linkcolor=black,      
citecolor=black,      
filecolor=black,      
urlcolor=blue]{hyperref}

\newtheorem{definition}{Definition}
\newtheorem{proof}{Proof}
\newtheorem{theorem}{Theorem}

\newdefinition{remark}{Remark}
\newtheorem{lemma}{Lemma}
\newtheorem{corollary}{Corollary}

\input{defs}

\usepackage{wrapfig}

 \biboptions{authoryear}

\begin{document}
\begin{frontmatter}

\title{VLM-RL: A Unified Vision Language Models and Reinforcement Learning Framework for Safe Autonomous Driving}

\author[a]{Zilin Huang\textsuperscript{†}}
\ead{zilin.huang@wisc.edu}
\author[a]{Zihao Sheng\textsuperscript{†}}
\ead{zihao.sheng@wisc.edu}
\author[b]{Yansong Qu\textsuperscript{†}}
\ead{qu120@purdue.edu}
\author[a]{Junwei You}
\ead{jyou38@wisc.edu} 
\author[a]{Sikai Chen\corref{cor1}}
\ead{sikai.chen@wisc.edu}

\cortext[cor1]{Corresponding author: Sikai Chen. {†} Equal contribution.}

\address[a]{Department of Civil and Environmental Engineering, University of Wisconsin-Madison, Madison, WI, 53706, USA}
\address[b]{Lyles School of Civil Engineering, Purdue University, West Lafayette, IN 47907, USA}

\begin{abstract}
In recent years, reinforcement learning (RL)-based methods for learning driving policies have gained increasing attention in the autonomous driving community and have achieved remarkable progress in various driving scenarios. However, traditional RL approaches rely on manually engineered rewards, which require extensive human effort and often lack generalizability. To address these limitations, we propose \textbf{VLM-RL}, a unified framework that integrates pre-trained Vision-Language Models (VLMs) with RL to generate reward signals using image observation and natural language goals. The core of VLM-RL is the contrasting language goal (CLG)-as-reward paradigm, which uses positive and negative language goals to generate semantic rewards. We further introduce a hierarchical reward synthesis approach that combines CLG-based semantic rewards with vehicle state information, improving reward stability and offering a more comprehensive reward signal. Additionally, a batch-processing technique is employed to optimize computational efficiency during training. Extensive experiments in the CARLA simulator demonstrate that VLM-RL outperforms state-of-the-art baselines, achieving a 10.5\% reduction in collision rate, a 104.6\% increase in route completion rate, and robust generalization to unseen driving scenarios. Furthermore, VLM-RL can seamlessly integrate almost any standard RL algorithms, potentially revolutionizing the existing RL paradigm that relies on manual reward engineering and enabling continuous performance improvements. The demo video and code can be accessed at: \href{https://zilin-huang.github.io/VLM-RL-website/}{\textcolor{magenta}{https://zilin-huang.github.io/VLM-RL-website/}}.
\end{abstract}
		
		
\end{frontmatter}

\section{Introduction}

Autonomous driving technology has made significant progress in recent years, yet achieving human-level safety and reliability remains a fundamental challenge \citep{feng2023dense,huang2024toward}. A particularly acute challenge is how to develop safe and generalizable driving policies for complex traffic environments \citep{di2021survey, cao2022trustworthy, he2024trustworthy, huang2024trustworthy, sheng2024ego, yao2023goal}. The rapid advancement of deep learning has catalyzed remarkable developments in this domain, particularly through imitation learning (IL) and reinforcement learning (RL) \citep{huang2024human, wu2024recent}, as shown in Fig.~\ref{fig1} (a). IL aims to learn driving policies by mimicking expert demonstrations, achieving impressive performance in controlled environments \citep{huang2024human}. Yet, IL-based methods face inherent limitations: they heavily depend on the scale and quality of demonstration data and exhibit poor generalization beyond the training distribution. In contrast, RL enables agents to actively learn optimal driving policies through direct interaction with the environment by maximizing carefully designed reward functions \citep{huang2024human}. The effectiveness of RL-enabled methods has been demonstrated in various decision-making scenarios, such as safe navigation \citep{mao2024integrating, he2024trustworthy}, car-following control \citep{hart2024towards}, trajectory control \citep{sheng2024traffic}, and lane change \citep{guo2024modeling}.

Nevertheless, a major challenge in applying RL is designing an appropriate reward function that will lead to the desired behavior \citep{ma2023eureka, venuto2024code, xietext2reward, wang2024rl}. While RL has been remarkably successful in domains where the reward function is clearly defined (e.g., gaming, robot manipulation), its application to autonomous driving remains troubled \citep{ye2024lord, hazra2024revolve, han2024autoreward, zhou2024context}. The fundamental difficulty is that the notion of \textit{``good driving"} encompasses complex, context-dependent behaviors, relying on tacit knowledge that is difficult to quantify and encode as a reward function \citep{ye2024lord}. This reflects Polanyi's paradox, which asserts that \textit{``we know more than we can tell"} \citep{polanyi2009tacit}. Traditionally, reward functions in the field of autonomous driving are usually manually designed based on expert intuition and heuristics, which often combine multiple sub-objectives such as speed maintenance, lane following, and collision avoidance \citep{chen2022interpretable, wang2023efficient, zhang2024chatscene}. However, this procedure, known as \textit{``reward engineering"}, requires considerable human effort and trial-and-error iterations \citep{abouelazm2024review}. As noted by \cite{han2024autoreward, knox2023reward, abouelazm2024review}, it faces several challenges such as expert knowledge dependence, multi-objective conflicts, and generalizability limitations.

Recent breakthroughs in foundation models, particularly large language models (LLMs) \citep{openai2023chatgpt} and vision-language models (VLMs) \citep{radford2021learning}, have demonstrated remarkable capabilities in understanding natural language instructions and complex visual scenes. This progress has inspired researchers to explore the use of these foundation models for reward shaping in RL, offering a promising solution to the longstanding challenge of reward design. The key idea is to leverage the rich semantic understanding capabilities of foundation models to translate human-specified goals into reward signals that can guide RL training effectively \citep{venuto2024code, ma2023eureka, xietext2reward}. In the field of robotics, integrating foundation models into reward functions has shown strong performance and promising generalization capabilities. Many experiments in this domain focus on tasks where the desired goal state is well-defined and easily specified, such as \textit{``Put carrot in bowl"} for manipulation tasks \citep{kim2024openvla}. This explicit goal specification allows researchers to harness pre-trained foundation models as zero-shot reward generators, simplifying the reward design process \citep{baumlivision, rocamonde2023vision, sontakke2024roboclip, fu2024furl}. However, while these approaches perform well across various robotic tasks, they encounter significant challenges in the context of safe autonomous driving. Unlike robotic manipulation tasks, where goals can be specified with high precision, driving objectives such as \textit{``drive safely"} are inherently abstract and context-dependent. These high-level instructions are difficult for foundation models to interpret effectively due to the complexity and variability of real-world driving scenarios \citep{ye2024lord}.

\begin{figure*}[t]
\centering
  \includegraphics[width=0.99999\textwidth]{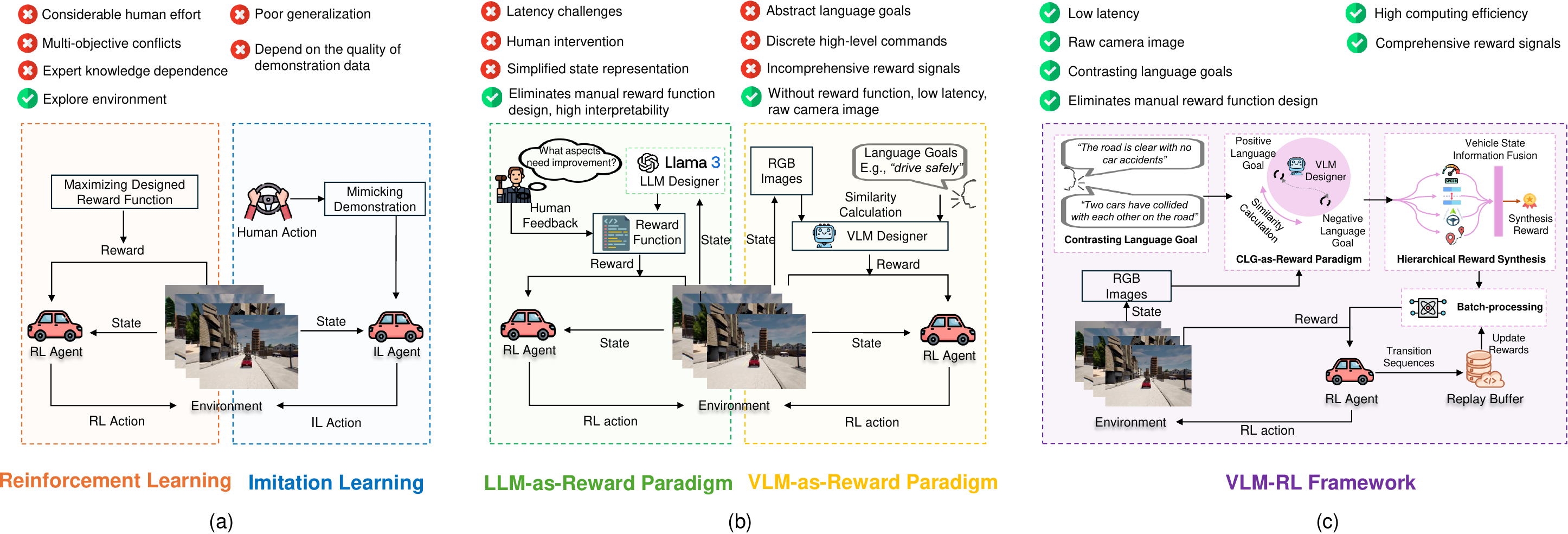}
  \caption{Comparative Overview of Reward Design Paradigms for Autonomous Driving. (a) Fundamentals and limitations of IL/RL-based methods for driving policy learning. (b) Fundamentals and limitations of foundation model-based reward design methods (i.e., LLM-as-Reward and VLM-as-Reward paradigms) for driving policy learning. (c) Our proposed VLM-RL framework, leverages VLMs to achieve a comprehensive and stable reward design for safe autonomous driving.}
  \label{fig1}
\end{figure*}

Several recent works have attempted to combine foundation models with RL for autonomous driving. Fig.~\ref{fig1} (b) illustrates two dominant paradigms for this integration: \textit{LLM-as-Reward} and \textit{VLM-as-Reward} paradigm. The first approach directly invokes LLMs to generate reward functions or codes, incorporating human feedback to iteratively refine the design. The second approach uses VLMs as reward functions, where the model evaluates the agent's state and provides immediate feedback based on the alignment with specified goals to guide behavior. Despite these advancements, existing works still face several critical limitations: (a) Most works rely solely on language descriptions to achieve desired behaviors \citep{zhou2024context, yildirim2024highwayllm}. While \cite{ye2024lord} proposes an opposite reward design, it focuses exclusively on negative scenarios and overlooks the rich semantic relationships between positive and negative driving behaviors. (b) Many approaches depend on real-time foundation model inference during deployment \citep{hazra2024revolve, han2024autoreward, zhou2024context}, introducing latency issues that are unacceptable for safety-critical driving decisions where real-time responsiveness is essential. (c) Current methods predominantly validate their effectiveness in simplified simulation environments such as HighwayEnv simulator \citep{yildirim2024highwayllm, ye2024lord, han2024autoreward, zhou2024context}. These methods use simplified state representations, which do not capture the complexity of real-world sensors (e.g., camera images) used in actual autonomous vehicles. (d) Many existing approaches generate discrete high-level commands (e.g., lane changes, acceleration) \citep{zhou2024context, ye2024lord}, which are insufficient for continuous and precise control needed in real-world vehicle operations.

Observing how humans learn new skills, we find that people typically learn more effectively through contrasting examples. For instance, when teaching someone to cook a steak, instructors often highlight both the correct and incorrect techniques: \textit{``A perfectly cooked steak has a golden-brown crust and a uniformly pink interior"} versus \textit{``If the steak turns completely dark brown and has a burnt smell, it is overcooked"}. This helps learners develop a comprehensive understanding of proper cooking techniques by recognizing both desired and undesired outcomes. Drawing inspiration from human learning, we propose a unified framework for integrating pre-trained VLMs and online RL, named \textbf{VLM-RL}, as shown in Fig.~\ref{fig1} (c). VLM-RL fundamentally rethinks how foundation models can be integrated into RL-based autonomous driving systems. We first introduce the concept of contrasting language goal (CLG) to guide RL-based safe driving tasks. Building upon the \textit{VLM-as-Reward}, a novel \textit{CLG-as-Reward} paradigm is then presented, leveraging CLG to generate more informative and context-aware rewards. To enhance learning stability, a hierarchical reward synthesis approach is adopted, combining CLG-based rewards with vehicle state information. These synthesized rewards are then integrated with standard RL for policy training. Additionally, a batch-processing technique is employed to improve computational efficiency during the training process. 

More importantly, VLM-RL implements a closed-loop end-to-end training pipeline that integrates camera-based perception with continuous control outputs, addressing a significant limitation of existing work that typically relies on simplified state representations and discrete actions. Our contributions can be summarized as follows:

\begin{itemize}
    \item We propose the VLM-RL framework, which leverages pre-trained VLMs as zero-shot reward models, eliminating the need for explicit reward engineering in RL-based safe driving tasks. \textbf{To our knowledge, VLM-RL is the first work in the autonomous driving field to unify VLMs with RL for end-to-end driving policy learning in the CARLA simulator.}
    \item We propose a novel CLG-as-Reward paradigm for reward shaping, which leverages pre-trained VLMs to generate semantic reward signals by measuring the semantic alignment between driving states and contrasting language descriptions (i.e., positive and negative language goals). 
    \item We propose a hierarchical reward synthesis approach that combines CLG-based rewards with vehicle state information to provide comprehensive and stable reward signals. This method addresses the limitations of using only coarse semantic rewards from VLMs, which can mislead policy optimization due to incomplete or imprecise behavior descriptions.
    \item We develop a batch-processing technique to ensure computational efficiency. Instead of calculating rewards immediately for every observation, batches of observations are periodically sampled from a replay buffer and processed through the pre-trained VLM. VLM-RL can be seamlessly integrated into almost any standard RL, enabling consistent performance improvements.
    \item We conducted extensive experiments in the CARLA simulator, demonstrating significant improvements in safety, efficiency, and generalization to diverse driving scenarios. Specifically, compared to state-of-the-art baselines, VLM-RL achieved 10.5\% reduction in collision rate, a 104.6\% increase in route completion rate, and successfully generalized to previously unseen scenarios without fine-tuning. 
\end{itemize}

The remainder of this paper is organized as follows. Section \ref{Related Works} reviews the related work. Section \ref{Preliminaries} introduces the preliminaries and problem formulation. Section \ref{Framework: VLM-RL} details the proposed VLM-RL framework. Section \ref{Experiments and Results} presents the experimental setup and results. Finally, Section \ref{Conclusions and Future Work} concludes the paper and outlines future research directions.

\section{Related Works}
\label{Related Works}

\subsection{Reward Design with Foundation Model}

The design of reward functions remains a fundamental challenge in RL. Recently, a new paradigm has emerged that leverages foundation models to generate reward signals for RL. \citet{kwon2023reward} first demonstrated the potential of LLMs, such as GPT-3 \citep{openai2023chatgpt}, in generating rewards for text-based tasks. Subsequent works extended this idea, showing that LLMs can generate structured code for robot training \citep{yu2023language} and Python code for various agents \citep{xietext2reward, ma2023eureka}. However, these methods often assume access to detailed environment information, which is challenging in autonomous driving. For instance, accurate data on surrounding vehicles' velocities and positions may not be available. In this work, the VLM-RL generates the reward signal directly from the visual input captured by the on-board camera, which does not require such assumptions. VLM-CaR \citep{venuto2024code} mitigates VLM query costs by breaking tasks into sub-objectives, though this is difficult for safe driving tasks. Other works use the embedding space of pre-trained VLMs, such as CLIP \citep{radford2021learning}. \cite{mahmoudieh2022zero} are the first to use fine-tuned CLIP as reward models for robotic manipulation. VLM-SR \citep{baumlivision} converts similarity-based rewards into binary rewards via thresholding, while RoboCLIP \citep{sontakke2024roboclip} compares task video embeddings to agent behavior. VLM-RM \citep{rocamonde2023vision} enhances rewards using goal-based baseline regularization, and RL-VLM-F \citep{wang2024rl} incorporates human preference labels for improved reward quality, and FuRL \citep{fu2024furl} addresses reward misalignment issues to refine reward signals further. These methods work well in robotics domains where goal states can be precisely defined and easily understood by VLMs. In contrast, autonomous driving involves inherently ambiguous language goal states that are difficult to define or verify. Additionally, robotics tasks typically involve static or controlled environments, whereas autonomous driving must deal with multiple agents and uncertain dynamic scenarios.

\subsection{Foundation Model in Autonomous Driving}

Recent breakthroughs in foundation models have inspired researchers to apply them to the field of autonomous driving, including scene understanding (e.g., DriveVLM \citep{tian2024drivevlm}, LeapAD \citep{mei2024continuously}), planning (e.g., DiLu \citep{wen2023dilu}, DriveMLM \citep{wang2023drivemlm}), scene generation (e.g., ChatScene \cite{zhang2024chatscene}, ChatSim \citep{wei2024editable}), human-vehicle interaction (e.g., DriVLMe \citep{huang2024drivlme}, Drive as you speak \citep{cui2024drive}), and end-to-end driving (e.g.,  LMDrive \citep{shao2024lmdrive}, DriveGPT4 \citep{xu2024drivegpt4}). Despite these advancements, leveraging foundation models for reward design in safe driving tasks has yet to be fully explored. LLM-RL \citep{zhou2024context} employs LLMs to intuitively shape reward functions via natural language prompts, enabling more human-like driving behavior. HighwayLLM \citep{yildirim2024highwayllm} integrates LLMs with RL to provide explainable decision-making in highway driving scenarios. REvolve \citep{hazra2024revolve} frames reward design as an evolutionary search problem, leveraging LLMs and human feedback to create human-aligned reward functions. In contrast, VLM-RL does not require human feedback. AutoReward \citep{han2024autoreward} utilizes LLMs to generate and refine reward functions through a closed-loop framework automatically. Most of the existing works heavily rely on real-time inference from foundation models, which may raise limitations such as latency issues. VLM-RL does not rely on direct model queries but instead utilizes their embedding spaces for reward computation. Perhaps the work closest to ours is LORD \citep{ye2024lord}, which uses undesired language goals to shape agent behavior. The key differences between VLM-RL and LORD are: (a) VLM-RL uses both desired and undesired goals combined with vehicle state information for richer reward signals; (b) VLM-RL uses camera-based visual inputs for more realistic perception; and (c) VLM-RL implements an end-to-end pipeline that produces continuous control outputs.

\section{Preliminaries}
\label{Preliminaries}
\subsection{Partially Observable Markov Decision Processes}
\label{Partially Observable Markov Decision Processes}
A partially observable Markov decision process (POMDP) is defined by the tuple \((S, \mathcal{A}, \theta, R, \mathcal{O}, \phi, \gamma, d_0)\), where \(S\) is the state space, \(\mathcal{A}\) is the action space, \(\theta(s' \mid s, a) : S \times S \times \mathcal{A} \to [0, 1]\) represents the transition function, \(R(s, a, s') : S \times \mathcal{A} \times S \to \mathbb{R}\) is the reward function, \(\mathcal{O}\) denotes the observation space, \(\phi(o \mid s) : S \to \Delta(\mathcal{O})\) is the observation distribution, and \(d_0(s) : S \to [0, 1]\) is the initial state distribution. At each timestep, the environment occupies a state \(s \in S\), and the agent selects an action \(a \in \mathcal{A}\). The environment transitions to a new state \(s'\) with probability \(\theta(s' \mid s, a)\). The agent then receives an observation \(o\) with probability \(\phi(o \mid s')\) and a reward \(r = R(s, a, s')\). A sequence of states and actions forms a trajectory \(\tau = (s_0, a_0, s_1, a_1, \ldots)\), where \(s_i \in S\) and \(a_i \in \mathcal{A}\). The return for a trajectory \(\tau\) is the discounted sum of rewards: \(g(\tau; R) = \sum_{t=0}^{T} \gamma^t R(s_t, a_t, s_{t+1})\). The agent's objective is to find a policy \(\pi(a \mid s)\) that maximizes the expected return \(G(\pi) = \mathbb{E}_\pi [g(\tau(\pi); R)]\).

\subsection{Vision-Language Models}
\label{Vision-Language Models}
VLMs have seen significant advancements in recent years \citep{you2024v2x}. These models are broadly defined as those capable of handling sequences of both language inputs \( l \in \mathcal{L}^{\leq n} \) and vision inputs \( i \in \mathcal{I}^{\leq m} \). In this context, \(\mathcal{L}\) represents a finite alphabet, and \(\mathcal{L}^{\leq n}\) refers to strings of length up to \(n\). Similarly, \(\mathcal{I}\) denotes the space of 2D RGB images, and \(\mathcal{I}^{\leq m}\) consists of image sequences of length up to \(m\). A notable class of pre-trained VLMs is CLIP \citep{radford2021learning}, which includes a language encoder \(\text{CLIP}_L: \mathcal{L}^{\leq n} \rightarrow \mathcal{V}\) and an image encoder \(\text{CLIP}_I: \mathcal{I} \rightarrow \mathcal{V}\), both mapping to a shared latent space \(\mathcal{V} \subseteq \mathbb{R}^k\). These encoders are trained jointly through contrastive learning on image-caption pairs. The training objective is to minimize the cosine distance between embeddings of matching pairs while maximizing it for non-matching pairs. CLIP has demonstrated strong performance in various downstream tasks and exhibits impressive zero-shot transfer capabilities \citep{rocamonde2023vision}.

\subsection{Problem Statement}
We model the task of training an autonomous driving agent as a POMDP, similar to \cite{rocamonde2023vision}. The agent's objective is to learn an optimal policy \(\pi : S \to A\) that maximizes the expected cumulative reward, expressed as \(H(\pi) = \mathbb{E} \left[ \sum_{t=0}^T \gamma^t R(s_t, a_t) \mid \pi \right]\). A key challenge in this context is to design an effective reward function $R(s, a, s')$ that guides the agent toward desirable behaviors. Traditional reward engineering requires manual specification of complex behaviors and constraints, which can be tedious, error-prone, and hard to generalize across diverse driving scenarios. Ideally, we wish to directly use VLMs to provide agents with rewards $R(s)$ to guide desired behaviors, as is done in the robotics domain. However, as mentioned earlier, using VLMs directly as rewards for autonomous driving still faces critical challenges. Our goal is to create a specialized VLM-as-Reward framework for the safe driving task to eliminate the need for explicit reward functions $G(\pi) = \mathbb{E} \left[ \sum_{t=0}^T \gamma^t R_{\text{VLM}}(s) \mid \pi \right]$.

\section{Framework: VLM-RL}
\label{Framework: VLM-RL}
In this section, we present a detailed description of VLM-RL framework. The framework addresses the fundamental challenge of reward design in autonomous driving by leveraging the semantic understanding capabilities of pre-trained VLMs (i.e., CLIP).

\begin{figure*}[t]
\centering
  \includegraphics[width=0.99999\textwidth]{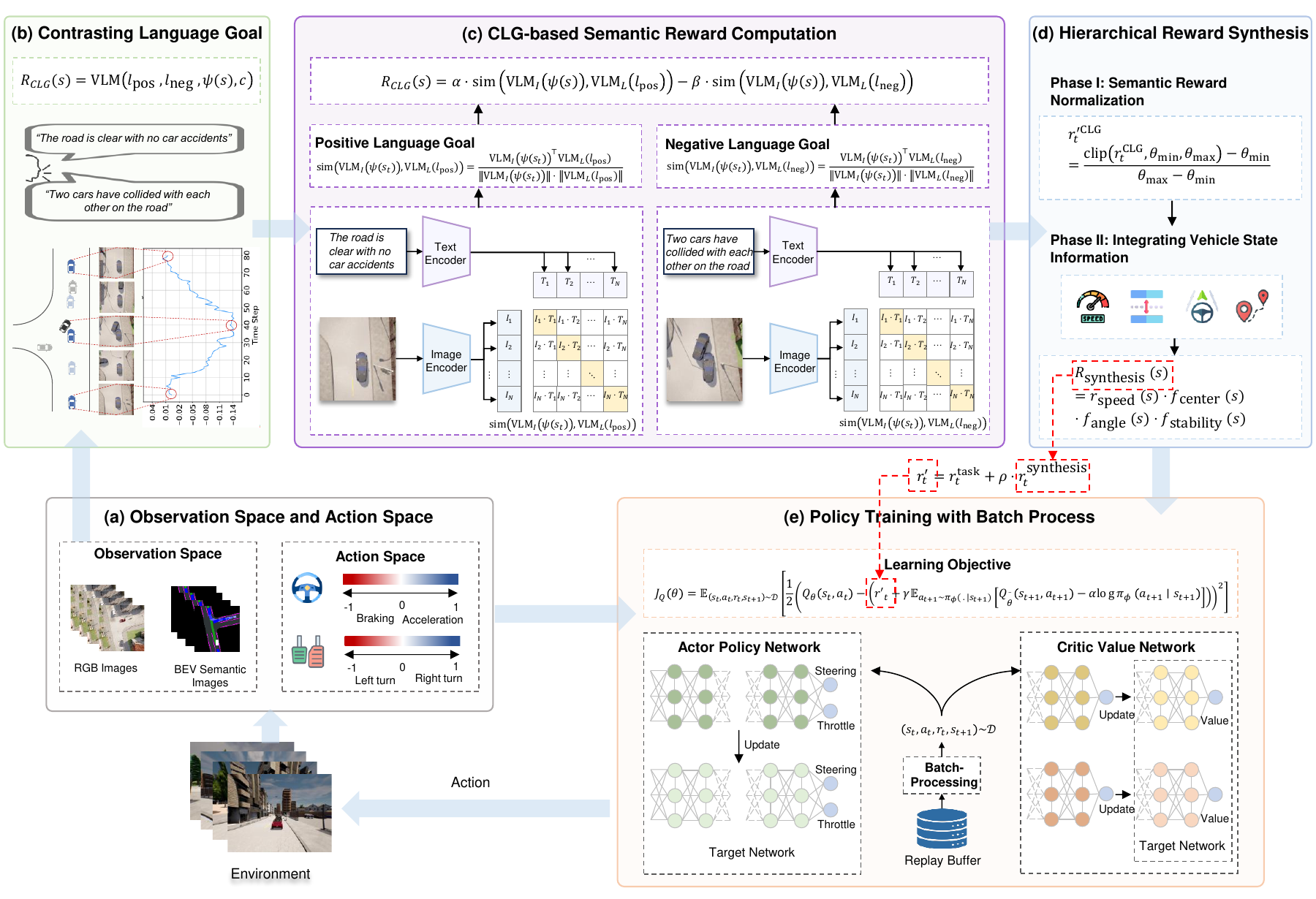}
  \caption{Architecture of the VLM-RL Framework for Autonomous Driving. (a) Observation and action spaces for policy learning; (b) Definition of CLG to provide semantic guidance; (c) CLG-based semantic reward computation using pre-trained VLMs; (d) Hierarchical reward synthesis that integrates semantic rewards with vehicle state information for comprehensive and stable reward signals; (e) Policy training with batch-processing, where SAC updates are performed using experiences stored in a replay buffer and rewards are computed asynchronously to optimize efficiency.}
  \label{fig2} 
\end{figure*}

\subsection{Overview}
The VLM-RL framework consists of four main components. First, we define the concept of CLG that describes both desired and undesired driving behaviors, providing a foundation for reward computation. Second, we utilize CLIP to compute semantic alignment between the current driving state and these contrasting language descriptions, generating semantic reward signals. Third, we develop a hierarchical reward synthesis approach that combines the semantic rewards with vehicle state information (e.g., speed, heading angle) to produce stable and comprehensive reward signals. Fourth, to optimize computational efficiency, we implement a batch-processing technique that periodically processes observations from the replay buffer rather than computing rewards in real time. Fig. \ref{fig2} illustrates the overall architecture of our framework. We describe each component in detail in the following subsections.

\begin{figure*}[t]
\centering
  \includegraphics[width=0.99999\textwidth]{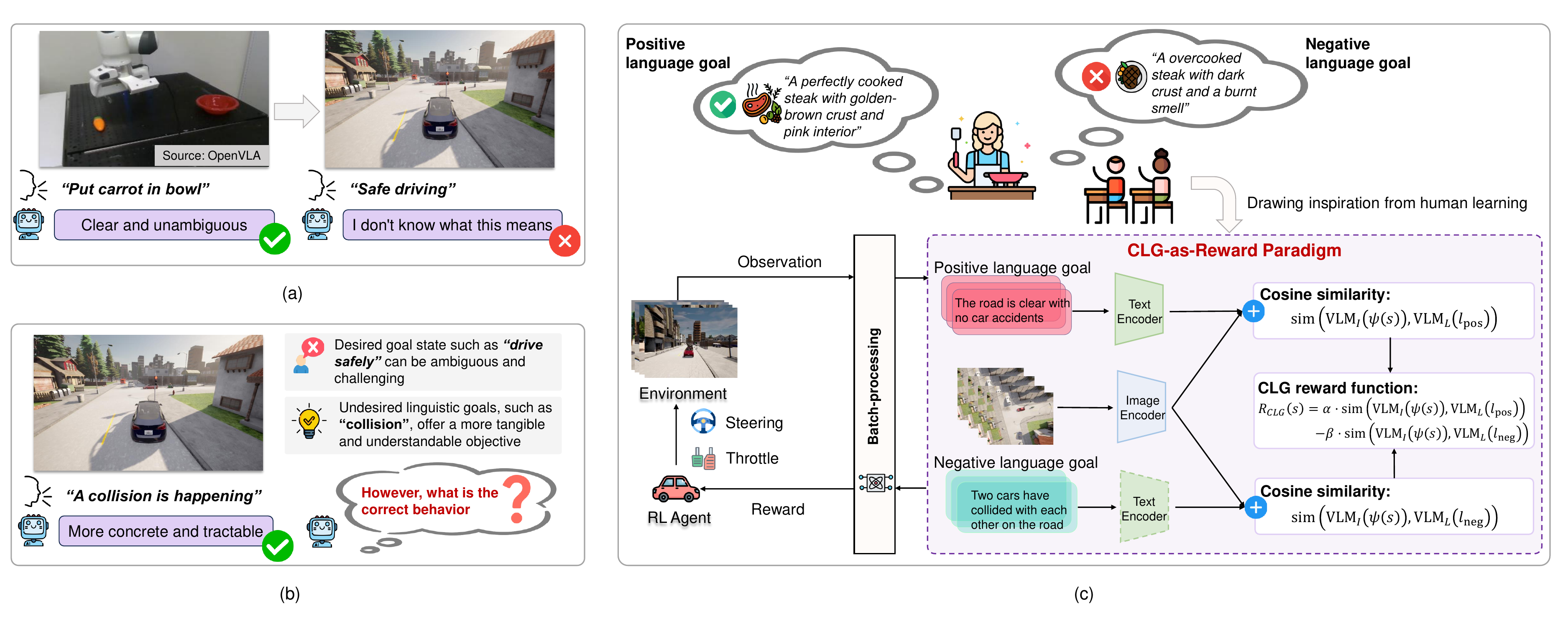}
  \caption{Conceptual comparisons of reward design paradigms. (a) Robotic manipulation tasks often feature well-defined goals (e.g., "Put carrot in bowl"), enabling VLMs to provide clear semantic rewards. (b) Existing methods that use only negative goals (e.g., "two cars have collided") focus on avoidance but lack positive guidance. (c) Our CLG-as-Reward paradigm integrates both positive and negative goals, allowing VLM-RL to deliver more informative semantic guidance for safer, more generalizable driving.}
  \label{fig3} 
\end{figure*}

\subsection{Contrasting Language Goal Definition}

Recent advances in robotics have demonstrated remarkable success in utilizing pre-trained VLMs as zero-shot reward models across diverse tasks \citep{ sontakke2024roboclip}. Given a task $\mathcal{T}$ and its natural language description $l \in \mathcal{L}^{\leq n}$, the fundamental approach involves leveraging VLMs to generate reward signals that guide the agent toward desired behaviors. This can be formally expressed as \citep{rocamonde2023vision}
\begin{equation}
R_{\text{VLM}}(s) = \text{VLM}(l, \psi(s), c)
\label{eq1}
\end{equation}
where $c \in \mathcal{L}^{\leq n}$ is an optional context that may include additional information or constraints. In this formulation, the VLM takes the language goal $l$, the current observation $\psi(s)$, and optional context $c$, and outputs a reward signal. 

In robotics, the success of this formulation relies on the ability to describe tasks and goal states with precise language. For example, in manipulation tasks (Fig. \ref{fig3} (a)), goals such as \textit{``Put carrot in bowl"} are clear and unambiguous, allowing VLMs to effectively measure progress by comparing state-goal relationships in their embedding space $\mathcal{V} \subseteq \mathbb{R}^k$. In contrast, autonomous driving poses unique challenges, as the goal of \textit{``Safe driving"} encompasses a wide range of acceptable behaviors and states. This abstract objective makes it difficult to establish clear semantic comparisons between the current vehicle state and the goal. While LORD \citep{ye2024lord} addresses this by using opposite language goals (Fig.\ref{fig3} (b)), this approach offers limited guidance by focusing only on states to avoid.

Drawing inspiration from human learning, where people often learn more effectively through contrasting goals, as exemplified by the steak-cooking scenario mentioned earlier, we propose using VLMs to generate semantic reward signals by aligning driving states with contrasting language descriptions (Fig.\ref{fig3} (c)). Specifically, we introduce the concept of CLG, which is defined as pairs of positive and negative descriptions that encapsulate desired and undesired driving behaviors. 

\begin{definition}[Contrasting Language Goal]
Given a driving task $\mathcal{T}$, we define contrasting language goal as a pair $(l_{\text{pos}}, l_{\text{neg}}) \in \mathcal{L}^{\leq n} \times \mathcal{L}^{\leq n}$, where $l_{\text{pos}}, l_{\text{neg}} \in \mathcal{L}^{\leq n}$ denote positive and negative language goals respectively. Positive goals describe desired outcomes (e.g., ``the road is clear with no car accidents"), while negative goals specify undesired scenarios (e.g., ``two cars have collided with each other on the road"). Based on Eq. (\ref{eq1}), the reward function of CLG is defined as
\begin{equation}
R_{\text{CLG}}(s) = \text{VLM}(l_{\text{pos}}, l_{\text{neg}}, \psi(s), c)
\label{eq2}
\end{equation}
\label{definition1}
\end{definition}

Specifically, we hope that the positive component will guide the agent toward a desirable state while the negative component will prevent the agent from entering an undesirable state. The ultimate goal is to provide more informative reward signals by encouraging desirable behaviors and punishing undesirable behaviors. Section \ref{CLG as Reward Paradigm} provides a detailed implementation of this idea.

\subsection{CLG-based Semantic Reward Computation}

\subsubsection{VLM as Rewards Revisited}

Safe driving tasks typically rely on sparse reward signals. In this setting, at each timestep $t$, given an observation $\psi(s_t)$ derived from state $s_t$, the agent executes an action $a_t \sim \pi_\theta(a_t | s_t)$ according to its policy $\pi_\theta$. The environment then provides a sparse task reward $r_t^{\text{task}}$, typically defined as $r_t^{\text{task}} = \delta_{\text{success}}$, meaning a reward of 1 is received only upon task success and otherwise the reward is 0 \citep{cao2022trustworthy}. Such sparse rewards present substantial challenges for RL training, as they provide limited learning signals across the majority of the state space. A common approach is to manually design dense reward signals by combining metrics such as speed and distance to waypoints, either through simple summation \citep{wang2023efficient} or weighted aggregation \citep{chen2022interpretable}. It is time-consuming, demands expertise, and may lead to conflicting sub-goals, resulting in suboptimal policies.

Building upon the general VLM reward formulation introduced in Eq. (\ref{eq1}), recent works \citep{rocamonde2023vision, baumlivision, fu2024furl} have proposed augmenting these sparse task rewards with VLM-generated rewards $r_t^{\text{VLM}}$. This hybrid reward formulation can be expressed as
\begin{equation}
r_t = r_t^{\text{task}} + \rho \cdot r_t^{\text{VLM}}
\label{eq3}
\end{equation}
where $\rho > 0$ is a weighting parameter that balances the relative importance of the VLM-generated reward against the sparse task reward.

\begin{definition} [VLM-as-Reward Paradigm]
Given the vision encoder $\text{VLM}_I: \mathcal{O} \rightarrow \mathcal{V}$ and language encoder $\text{VLM}_L: \mathcal{L}^{\leq n} \rightarrow \mathcal{V}$ that map into the same latent space $\mathcal{V} \subseteq \mathbb{R}^k$, and a sequence of state-action transitions $\{ s_t, a_t, r_t, s_{t+1} \}_{t=1}^{T}$, the VLM reward is defined as
\begin{equation}
\label{eq4}
r_t^{\text{VLM}} = D\left(\text{VLM}_I(\psi(s_t)), \text{VLM}_L(l)\right)
\end{equation}
where $D: \mathcal{V} \times \mathcal{V} \rightarrow \mathbb{R}$ is a distance metric between the embedded representations. Most works adopt CLIP \citep{radford2021learning} as the pre-trained VLM, employing cosine similarity as the distance metric \citep{rocamonde2023vision, sontakke2024roboclip}
\begin{equation}
\label{eq5}
r_t^{\text{VLM}} = \text{sim}\left(\text{VLM}_I(\psi(s_t)), \text{VLM}_L(l)\right) = \frac{\text{VLM}_I(\psi(s_t))^\top \text{VLM}_L(l)}
{\Vert \text{VLM}_I(\psi(s_t)) \Vert \cdot \Vert \text{VLM}_L(l) \Vert}
\end{equation}
where $\text{sim}(\cdot, \cdot)$ denotes the cosine similarity between embeddings, and $\text{VLM}_I$ and $\text{VLM}_L$ are the vision and language encoders, respectively. In this case, we don't need the context $c$. Here, the language goals $l$ typically express desired behaviors. The computation process involves three steps: (a) The observation $\psi(s)$ is processed through the $\text{VLM}_I$ to obtain a state embedding in the shared latent space $\mathcal{V}$. (b) The language goals are encoded via $\text{VLM}_L$ to obtain their respective embeddings in the same space. (c) The reward is computed as the cosine similarity between the state embedding and the goal embedding.
\label{definition2}
\end{definition} 

While this formulation works well for robotics tasks with concrete goals, autonomous driving lacks such well-defined objectives. Recent work LORD \citep{ye2024lord} proposes using opposite language goals for safe driving tasks. Their key insight is that undesired states (e.g., collisions) are often more concrete and easier to specify than desired states. By defining the reward function $R_{\text{LORD}}(s)$ as
\begin{equation}
\label{eq6}
R_{\text{LORD}}(s) = 1 - \text{sim}\left(\text{VLM}_I(\psi(s)), \text{VLM}_L(l_{\text{neg}})\right)
\end{equation}

The formulation aims to minimize the similarity between the current observation and the negative goal, thereby penalizing undesirable behaviors. By focusing solely on avoiding negative states, this approach may lack positive guidance for desirable behaviors, potentially limiting learning efficiency.

\subsubsection{CLG as Reward Paradigm}
\label{CLG as Reward Paradigm}
Following the VLM-as-Reward paradigm in Definition \ref{definition2}, we propose a novel CLG-as-Reward paradigm specifically designed for safe driving tasks.

\begin{definition} [CLG-as-Reward Paradigm]
Given the CLG $(l_{\text{pos}}, l_{\text{neg}})$ introduced in Definition \ref{definition1}, we define the CLG reward function as
\begin{equation}
\label{eq7}
R_{\text{CLG}}(s) = \ \alpha \cdot \text{sim}\bigl(\text{VLM}_I(\psi(s)), \text{VLM}_L(l_{\text{pos}})\bigr) - \beta \cdot \text{sim}\bigl(\text{VLM}_I(\psi(s)), \text{VLM}_L(l_{\text{neg}})\bigr)
\end{equation}
where $\alpha, \beta > 0$ are weighting factors satisfying $\alpha + \beta = 1$. If $\alpha > \beta$, the agent focuses more on achieving the positive goal, while if $\alpha < \beta$, the agent emphasizes steering clear of negative outcomes. For simplicity, in this work, we set $\alpha = \beta = 0.5$, i.e., the two goals are equally prioritized. $\text{sim}(\cdot, \cdot)$ denotes the cosine similarity between embeddings as defined in Eq. (\ref{eq5}).
\label{definition3}
\end{definition}

This formulation ensures that the agent is guided by both the positive and negative goals simultaneously. In other words, it encourages the agent to seek states similar to the positive goal while avoiding states similar to the negative goal, offering more informative guidance for policy learning. The following Thm. \ref{theorem:clg_advantage} formally establishes the effectiveness of the CLG-as-Reward paradigm.

\begin{theorem}[Effectiveness of CLG-as-Reward Paradigm]
\label{theorem:clg_advantage}
Assume the VLM embeddings accurately capture the semantic content of observations and language goals. Under this assumption, optimizing the policy $\pi$ to maximize the CLG reward $R_{\text{CLG}}$ defined in Eq.~(\ref{eq7}) encourages the agent to simultaneously increase similarity to the positive goal and decrease similarity to the negative goal. As a result, the learned policy not only achieves the desired driving behaviors but also avoids undesirable ones. 
\end{theorem}

\begin{proof}
The agent aims to maximize the expected discounted return:
\begin{equation}
J(\pi) = \mathbb{E}_{\pi}\left[ \sum_{t=0}^{T} \gamma^t R_{\text{CLG}}(s_t) \right]
\label{eq8}
\end{equation}

An increase in $\text{sim}(\text{VLM}_I(\psi(s_t)), \text{VLM}_L(l_{\text{pos}}))$ indicates that the current state $s_t$ more closely aligns with the positive goal. Since $R_{\text{CLG}}$ adds a term proportional to this similarity, states that resemble the positive goal yield higher rewards. Conversely, $R_{\text{CLG}}$ subtracts a term proportional to $\text{sim}(\text{VLM}_I(\psi(s_t)), \text{VLM}_L(l_{\text{neg}}))$, meaning states similar to the negative goal reduce the reward. Thus, maximizing $R_{\text{CLG}}$ naturally pushes the agent toward states that are semantically closer to the positive goal and farther from the negative goal.

Formally, let $s_t$ and $s'_t$ be two potential subsequent states with embeddings $\mathbf{v}_t = \text{VLM}_I(\psi(s_t))$ and $\mathbf{v}'_t = \text{VLM}_I(\psi(s'_t))$. Define $\mathbf{v}_{\text{pos}} = \text{VLM}_L(l_{\text{pos}})$ and $\mathbf{v}_{\text{neg}} = \text{VLM}_L(l_{\text{neg}})$. If
\begin{equation}
\text{sim}(\mathbf{v}_t, \mathbf{v}_{\text{pos}}) - \text{sim}(\mathbf{v}_t, \mathbf{v}_{\text{neg}})
> \text{sim}(\mathbf{v}'_t, \mathbf{v}_{\text{pos}}) - \text{sim}(\mathbf{v}'_t, \mathbf{v}_{\text{neg}})
\label{eq9}
\end{equation}
then $R_{\text{CLG}}(s_t) > R_{\text{CLG}}(s'_t)$. The agent, through repeated interaction and policy updates, will tend to choose actions leading to $s_t$ rather than $s'_t$, as $s_t$ yields higher expected returns.

\end{proof}

Over time, this consistent preference ensures the learned policy converges toward behaviors that improve positive similarity while reducing negative similarity. Incorporating both positive and negative goals thus provides a more informative learning signal than using either one alone, resulting in superior policy learning. Furthermore, as demonstrated theoretically in \ref{appendices1}, the CLG-as-Reward paradigm enhances the robustness of the learned policy, making it more resilient to uncertainty and adversarial perturbations.

\subsection{Hierarchical Reward Synthesis}

In this work, we follow the standard VLM-as-Reward paradigm, i.e., using only a language description of the task \citep{rocamonde2023vision, sontakke2024roboclip, wang2024rl}. Yet, as noted by \cite{fu2024furl}, while zero-shot VLMs are effective in capturing coarse semantics, they often fall short in accurately representing fine-grained details. Furthermore, a single language description cannot comprehensively capture all the nuances of desired driving behaviors. As a result, relying solely on semantic rewards $R_{\text{CLG}}$ could potentially mislead policy optimization in complex driving scenarios. Previous work has explored various strategies to address this issue: LAMP \citep{adeniji2023language} uses VLM-based reward for behavior pre-training, ZSRM \citep{mahmoudieh2022zero} retrains VLMs with task-specific datasets, and FuRL \citep{fu2024furl} fine-tunes VLM representations and uses relay RL technique. 

In contrast to these approaches, we aim to preserve the zero-shot capability of VLMs by integrating vehicle state information, which is readily available from on-board sensors, to generate more stable and comprehensive reward signals. In detail, We propose a hierarchical reward synthesis approach consisting of two phases: (a) generating normalized semantic rewards from VLMs and (b) combining these semantic rewards with vehicle state information to produce the synthesis reward signal.

\textbf{Phase I: Semantic Reward Normalization.} First, we compute the semantic rewards $r_t^{\text{CLG}}$ by processing batches of observation frames through the CLIP. To ensure stability, we normalize the similarity scores to the range $[0, 1]$:
\begin{equation}
r'{_t^{\text{CLG}}} = \frac{\text{clip}(r_t^{\text{CLG}}, \theta_{\text{min}}, \theta_{\text{max}}) - \theta_{\text{min}}}{\theta_{\text{max}} - \theta_{\text{min}}}
\label{eq10}
\end{equation}
where $\theta_{\text{min}}$ and $\theta_{\text{max}}$ are empirically set to -0.03 and 0.0, respectively, to avoid extreme values and ensure consistent scaling. $\text{clip}(x, a, b)$ constrains $x$ within the interval $[a, b]$.

\textbf{Phase II: Integrating Vehicle State Information.} We incorporate vehicle state information to produce the synthesis reward signal. This step leverages on-board sensor data to ensure the reward captures realistic driving behavior and safety constraints.

\begin{definition}[Synthesis Reward Function]
The synthesis reward function $R_{\text{synthesis}}: \mathcal{S} \rightarrow \mathbb{R}$ is computed by combining the normalized semantic reward $r'{_t^{\text{CLG}}}$ with vehicle state information. Specifically: 
\begin{equation} 
R_{\text{synthesis}}(s) = r_{\text{speed}}(s) \cdot f_{\text{center}}(s) \cdot f_{\text{angle}}(s) \cdot f_{\text{stability}}(s)
\label{eq11}
\end{equation} 
where $r_{\text{speed}}$ modulates speed alignment, computed as $r_{\text{speed}} = 1 - \frac{|v - v_{\text{target}}|}{v_{\text{max}}}$ with $v_{\text{target}} = r'{_t^{\text{CLG}}} \cdot v_{\text{max}}$. $f_{\text{center}}(s)$ evaluates the vehicle's lateral position relative to the lane center. $f_{\text{angle}}(s)$ measures the vehicle's orientation with respect to the road direction. $f_{\text{stability}}(s)$  accounts for the consistency of the vehicle's lateral position relative to the lane center. Each term is bounded within [0,1].
\end{definition}

Compared to traditional weighted-sum reward designs \citep{chen2022interpretable, wang2023efficient}, this multiplicative formulation naturally captures the interdependence of safety criteria without extensive parameter tuning. It yields an interpretable, stable, and easily implementable reward structure that leverages both semantic guidance from the VLM and actionable, high-fidelity vehicle state signals. The workflow of the hierarchical reward synthesis is shown in 
\ref{appendices2} as pseudocode. We also demonstrate the convergence and stability of the synthesis reward function in \ref{appendices3} and \ref{appendices4}.

Now, by combining the synthesis  reward function in Eq.~(\ref{eq11}) with Eq.~(\ref{eq3}), we obtain the final reward function for the VLM-RL framework:
\begin{equation}
r'_t = r_t^{\text{task}} + \rho \cdot r_t^{\text{synthesis}}
\label{eq12}
\end{equation}

This formulation allows the agent to benefit from both explicit task success signals and dense, context-aware rewards. The sparse task reward $r_t^{\text{task}}$ ensures that the agent remains goal-oriented, while the synthesis reward $r_t^{\text{synthesis}}$ provides continuous feedback based on both high-level semantic understanding and low-level vehicle dynamics.

\subsection{Policy Training with Batch-Processing}

We adopt the soft actor-critic (SAC) algorithm \citep{haarnoja2018soft} as the backbone RL framework, due to its superior sample efficiency and stability in continuous control tasks. The SAC algorithm aims to maximize the expected return while encouraging exploration through entropy regularization. The objective can be written as:
\begin{equation} 
J(\pi_\phi) = \mathbb{E}{\pi_\phi} \left[ \sum_{t=0}^{T} \gamma^t \bigl( R(s_t,a_t) + \alpha \mathcal{H}(\pi_\phi(\cdot|s_t)) \bigr) \right]
\label{eq13}
\end{equation}
where $\gamma \in [0,1)$ is the discount factor, $\alpha > 0$ is the entropy temperature parameter controlling the trade-off between return and entropy maximization, and $\mathcal{H}(\pi_\phi(\cdot|s_t))$ is the entropy of the policy at state $s_t$.

To update the policy parameters $\phi$, SAC minimizes the following objective:
\begin{equation} 
\label{eq14}
J_\pi(\phi) = \mathbb{E}_{s_t \sim \mathcal{D}} \left[ \mathbb{E}_{a_t \sim \pi_\phi(\cdot|s_t)} \left( -Q_\theta(s_t,a_t) + \alpha \log \pi_\phi(a_t|s_t) \right) \right]
\end{equation}
where $\mathcal{D}$ is the replay buffer, and $Q_\theta$ is the Q-function parameterized by $\theta$.

The Q-function parameters $\theta$ are updated by minimizing the soft Bellman residual:
\begin{equation}
\label{eq15}
J_Q(\theta) = \mathbb{E}_{(s_t,a_t,r_t,s_{t+1}) \sim \mathcal{D}} \left[ \frac{1}{2} \biggl( Q_\theta(s_t,a_t) - \bigl( r_t + \gamma \mathbb{E}_{a_{t+1}\sim\pi_\phi(\cdot|s_{t+1})} [Q_{\bar{\theta}}(s_{t+1},a_{t+1}) - \alpha \log \pi_\phi(a_{t+1}|s_{t+1})] \bigr) \biggr)^2 \right] 
\end{equation}
where $Q_{\bar{\theta}}$ is a target Q-function with periodically updated parameters $\bar{\theta}$.

Here, we replace the standard reward $r_t$ in the soft Bellman residual with $r'_t$ defined in Eq. (\ref{eq12}):
\begin{equation}
\label{eq16}
J_Q(\theta) = \mathbb{E}_{(s_t,a_t,r_t,s_{t+1}) \sim \mathcal{D}} \left[ \frac{1}{2} \biggl( Q_\theta(s_t,a_t) - \bigl( r'_t + \gamma \mathbb{E}_{a_{t+1}\sim\pi_\phi(\cdot|s_{t+1})} [Q_{\bar{\theta}}(s_{t+1},a_{t+1}) - \alpha \log \pi_\phi(a_{t+1}|s_{t+1})] \bigr) \biggr)^2 \right] 
\end{equation}

During training, the critic networks learn to estimate future returns based on Eq. (\ref{eq16}), while the policy network learns to maximize these returns through the standard SAC policy gradient updates.

To address the computational bottleneck of CLIP inference, we develop a batch-processing technique. During environment interaction, tuples of $(o_t, s_t, a_t, r_t, o_{t+1}, s_{t+1})$ are stored in a replay buffer. Here, $o_t$ represents the raw observation image required for CLIP processing, and $s_t$ contains the processed state information for policy learning. At predefined intervals, we sample a batch of observations from the replay buffer and process them through the CLIP encoder. The CLIP embeddings of the CLG ($l_{\text{pos}}$ and $l_{\text{neg}}$) are computed only once at the start of training, as they remain constant. We compute the synthesized rewards according to Eq.(\ref{eq11}), which then are used to update the stored transitions in the replay buffer. The SAC algorithm subsequently samples these updated transitions via its standard update procedure for policy optimization. This approach effectively decouples the computationally expensive reward computation from the main RL training loop, enabling the agent to continue learning while rewards are computed asynchronously. The complete training procedure is outlined in \ref{appendices5}.

\section{Experiments and Results}
\label{Experiments and Results}

\subsection{Experiment Setting}
\subsubsection{RL Setups}
The RL agent takes three types of inputs: 
(1) a bird's-eye view (BEV) semantic segmentation image that captures the surrounding environment, including drivable areas, lane boundaries, and other traffic participants, as illustrated in Fig. \ref{bev} (c)-(e). This provides crucial spatial information for navigation and obstacle avoidance. 
(2) ego state information consisting of the current steering angle, throttle value, and vehicle speed. These values reflect the vehicle's dynamic state and are essential for maintaining smooth control. 
(3) future navigation information represented by the next 15 waypoints along the planned route. Each waypoint is defined by its $(x,y)$ coordinates relative to the vehicle's current position, helping the agent understand and follow the desired trajectory. 

The action space is designed as a continuous 2-dimensional space $[-1,1]^2$, where each dimension controls different aspects of vehicle motion. The first dimension corresponds to the steering angle, with values in $[-1,1]$ representing the full range of steering control. Specifically, $-1$ indicates maximum left turn, $0$ represents straight ahead, and $+1$ indicates maximum right turn. The second dimension combines throttle and brake control in a single value range $[-1,1]$. When this value is positive ($[0,1]$), it directly maps to the throttle intensity, with $1$ representing full throttle. Conversely, when the value is negative, its absolute value maps to brake intensity, where $-1$ corresponds to full brake.
An episode is terminated when any of the following conditions are met: (a) collision with any obstacles, vehicles, or pedestrians, (b) deviation from the road center line by more than 3 meters, or (c) vehicle speed remains below 1 km/h for more than 90 consecutive seconds, indicating the agent is stuck or unable to progress. These termination conditions are designed to enforce safe driving behavior and ensure efficient navigation progress.

We build our implementation upon the Stable-Baselines3 library \citep{raffin2021stable}, which provides reliable and well-tested implementations of modern RL algorithms. 
Stable-Baselines3 offers a modular design and stable performance, allowing us to focus on extending the core algorithms rather than implementing them from scratch. 
Specifically, we extend the standard implementations of SAC and PPO to incorporate our CLG-based and hierarchical reward computation during the training process. 
The policy network architecture is specifically designed for processing heterogeneous input types: we employ a 6-layer CNN to extract features from the BEV semantic segmentation images, while using MLPs to process both the ego state information and future navigation waypoints. 
These processed features are then concatenated before being fed into the final policy head for action prediction.

\subsubsection{Driving Scenarios}

We train all models in CARLA's Town 2 map to ensure a fair comparison and evaluate the effectiveness of our approach and all baseline models. 
As shown in Fig. \ref{towns} (b), this town presents a typical European-style urban layout with a variety of challenging driving scenarios. 
It consists of several interconnected areas including a residential district, a commercial zone with single-lane roads, and complex intersections controlled by traffic lights. 
The compact nature of Town 2 makes it particularly suitable for evaluation, as it provides diverse driving conditions within a manageable scale, including both straight roads and curved segments, multiple T-junctions, and different types of lane markings and road geometries. These features create challenging scenarios for assessing both basic driving capabilities and complex decision-making behaviors.
In Section \ref{sec5.7.1}, we further evaluate the generalization ability of our approach in Towns 1, 3, 4, and 5, as shown in Figs. \ref{towns} (a), (c), (d), and (e), respectively.
Unless otherwise specified, all results are reported based on experiments conducted in Town 2.

To create a more realistic and challenging environment, we populate the town with 20 vehicles running in autopilot mode. 
These vehicles are randomly spawned across the map and operate using CARLA's built-in traffic manager, which enables them to follow traffic rules, respond to traffic lights, and perform basic collision avoidance. 
This dynamic traffic flow significantly increases the complexity of the learning task for our RL agent, as it must now handle various interactive scenarios such as car following, overtaking, and yielding to other vehicles. 
The presence of multiple moving vehicles not only makes the environment more similar to real-world urban driving conditions but also challenges the RL agent to develop more robust and adaptive driving strategies.

\begin{figure}
  \centerline{\includegraphics[width=0.993\textwidth]{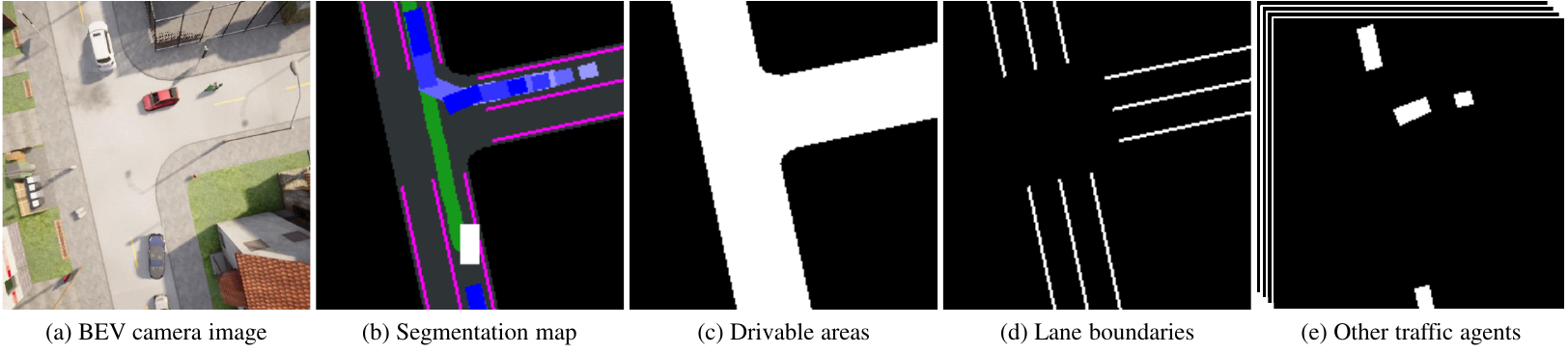}}
  \caption{Bird's eye view of RL agent's surrounding environment, where the purple vehicle in (a) and the white box in (b) represents the RL agent.}
  \label{bev}
\end{figure}

\begin{figure}
  \centerline{\includegraphics[width=0.993\textwidth]{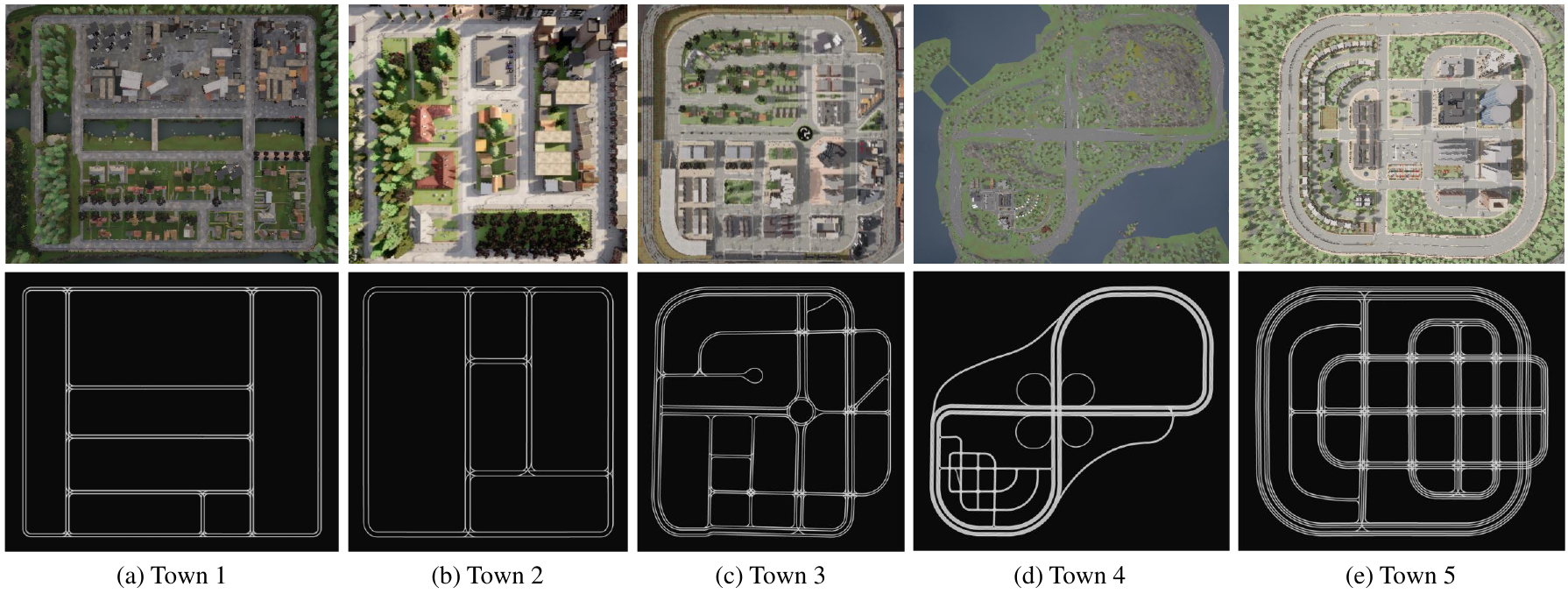}}
  \caption{Bird's eye view of Towns and their drivable routes in CARLA.}
  \label{towns}
\end{figure}

\subsubsection{Navigation Routes}
We dynamically assign navigation routes to the RL agent during training and evaluation. 
At each reset, we utilize the 101 predefined spawn points available on the drivable routes in Fig. \ref{towns} (b) as potential starting and destination locations. 
Specifically, we randomly select two distinct spawn points to serve as the start and end points, then employ the A* search algorithm to compute the shortest path between them, which becomes the navigation route for the agent. 
Notably, instead of terminating the episode upon reaching the destination, we continuously generate new navigation routes for the agent by repeating this random selection and path planning process. 
This dynamic route assignment continues until the cumulative driving distance within the episode reaches 3000 meters, allowing us to evaluate the agent's performance across diverse navigation scenarios in a single episode.

\subsubsection{CLIP Config}

We employ the CLIP model \citep{radford2021learning} as our foundational VLM for CLG-based semantic reward generation. 
Specifically, we utilize OpenCLIP's ViT-bigG-14 model pre-trained on the LAION-2B dataset with 2.32 English billion image-text pairs \citep{schuhmann2022laion}.
The model applies a patch size of 14$\times$14 pixels and accepts images with a resolution of 224$\times$224 pixels as input, which we obtain by resizing the original CARLA camera images. 
During inference, we leverage CLIP's visual encoder to extract high-dimensional feature representations from the driving scenes, while the text encoder processes our predefined CLG. All CLIP components are kept frozen during our experiments to maintain stable and consistent semantic reward generation.

\subsection{Evaluation Metrics}
To comprehensively evaluate the performance and safety aspects of our autonomous driving system, we employ multiple quantitative metrics that assess both driving efficiency and safety characteristics. For driving efficiency assessment, we measure the average speed (\textbf{AS}) maintained by the vehicle throughout episodes, the route completion (\textbf{RC}) which represents the number of successfully completed routes during one episode, and the total traveled distance (\textbf{TD}) which captures the cumulative distance covered by the vehicle during each episode.

Safety performance is evaluated through several complementary metrics. The fundamental collision rate (\textbf{CR}) measures the percentage of episodes containing collision events. We further analyze collision patterns through two frequency metrics: time-based collision frequency (\textbf{TCF}), measuring collisions per 1000 time steps, and distance-based collision frequency (\textbf{DCF}), measuring collisions per kilometer traveled. 
To assess collision severity, we record the collision speed (\textbf{CS}) at the moment of each collision. 
Additionally, we track the inter-collision time steps (\textbf{ICT}), which measure the average number of time steps between consecutive collision events, providing insights into the temporal distribution of safety incidents.
In the test phase, we also report the success rate (\textbf{SR}) to evaluate the model's ability to successfully reach the destination across 10 predefined routes.

\subsection{Baselines}
We compare our method against state-of-the-art baselines, which can be categorized into two primary groups: expert-designed reward methods and LM-designed reward methods.

\paragraph{Expert-designed Reward Methods} 
We implement the following baselines with manually designed reward functions using both SAC and PPO. These methods include binary rewards that only consider collision states, and summation rewards that combine multiple weighted terms to guide driving behavior:
\begin{itemize}
    \item \textbf{TIRL}~\citep{cao2022trustworthy} employs a simple binary reward that only penalizes collision states with -1 and assigns 0 reward to all other states.
    \item \textbf{Chen-SAC}~\citep{chen2022interpretable} employs a reward function that penalizes collisions, speeding, running out of lane, and excessive steering, while incentivizing forward velocity and controlled lateral acceleration to guide autonomous driving decisions.
    \item \textbf{ASAP-RL-PPO}~\citep{wang2023efficient} employs a reward function that provides positive incentives for forward progress, reaching the destination, and overtaking vehicles, while applying penalties for collisions with other vehicles or road curbs.
    \item \textbf{ChatScene}~\citep{zhang2024chatscene} designs a weighted sum of reward terms that encourages smooth driving behaviors (longitudinal speed, lateral acceleration, and steering control) while penalizing unsafe actions (collisions, out-of-lane driving, and speeding). A small constant reward is added as baseline incentive to facilitate learning.
\end{itemize}

\paragraph{LM-designed Reward Methods}
We also compare against recent approaches that leverage language models for reward design, including both LLM-based and VLM-based methods:
\begin{itemize}
    \item \textbf{Revolve} and \textbf{Revolve-auto}~\citep{hazra2024revolve} introduce an evolutionary framework that utilizes LLMs to generate reward function code based on human feedback. For simplicity, we adopt their best-performing reward function provided in the paper for comparison.
    \item \textbf{VLM-SR}~\citep{baumlivision} uses off-the-shelf VLMs like CLIP to generate binary reward signals by calculating the cosine similarity between image observations and language goals, followed by softmax normalization and thresholding to determine goal achievement.
    \item \textbf{RoboCLIP}~\citep{sontakke2024roboclip} generates sparse reward signals at the end of each episode by computing the similarity between video observations of agent trajectories and a task descriptor. Due to the complexity of autonomous driving, we adapt this baseline by using CLIP to compute the similarity between each frame's image observation and the language goal, generating a dense reward signal.
    \item \textbf{VLM-RM}~\citep{rocamonde2023vision} uses a baseline-target approach to project the current state embedding onto the direction between a baseline state (e.g., ``a car") and a target state (e.g., ``a car is driving safely").
    \item \textbf{LORD}~\citep{ye2024lord} focuses on negative reward generation through VLMs to penalize unsafe driving behaviors, using concrete undesired states like ``collision" to shape the reward signal. We also adapt this baseline by incorporating a dense speed-based reward, denoted as \textbf{LORD-Speed}, to further guide the agent's behavior based on driving speed.
\end{itemize}

\begin{table*}[!t]
  \begin{small}
  \caption{Performance comparison with baselines during training. Mean and standard deviation over 3 seeds. The best results are marked in \textbf{bold}.}
  \label{tab1}
  \centering
  \renewcommand{\arraystretch}{1.5}
  \setlength{\tabcolsep}{4pt}
  \begin{tabular}{@{}cccccccccc@{}}
  \toprule
  Model & Reference & AS~$\uparrow$ & RC~$\uparrow$ & TD~$\uparrow$ & CS~$\downarrow$ & CR~$\downarrow$ & ICT~$\uparrow$ & DCF~$\downarrow$ & TCF~$\downarrow$ \\
  \midrule 
  
  \rowcolor{gray!15}
  \multicolumn{10}{l}{\textit{\textbf{Expert-designed Reward Methods (Binary Rewards) }}} \\
  TIRL-SAC & TR-C'22      & 0.01 {\tiny $\pm$ 0.01} & 0.01 {\tiny $\pm$ 0.003} & 0.21 {\tiny $\pm$ 0.13} & 3.0 {\tiny $\pm$ 1.0}  & \textbf{0.013} {\tiny $\pm$ 0.01} & \textbf{54965} {\tiny $\pm$ 15734} & 14009 {\tiny $\pm$ 3864} & 14.0 {\tiny $\pm$ 0.90}  \\
  TIRL-PPO & TR-C'22      & 0.26 {\tiny $\pm$ 0.25} & 0.04 {\tiny $\pm$ 0.23}  & 2.95 {\tiny $\pm$ 1.46} & 2.9 {\tiny $\pm$ 0.8}  & 0.12 {\tiny $\pm$ 0.07}  & 32507 {\tiny $\pm$ 16076} & 3290 {\tiny $\pm$ 2368} & 5.50 {\tiny $\pm$ 2.91}  \\
  \midrule

  \rowcolor{gray!15}
  \multicolumn{10}{l}{\textit{\textbf{Expert-designed Reward Methods (Summation Rewards) }}} \\
  Chen-SAC & T-ITS'22 & \textbf{19.9} {\tiny $\pm$ 1.07} &  0.68 {\tiny $\pm$ 0.12} & 147.9 {\tiny $\pm$ 32.2} &  9.14 {\tiny $\pm$ 3.66} & 0.12 {\tiny $\pm$ 0.04}  &  2325 {\tiny $\pm$ 518} & 20.3 {\tiny $\pm$ 4.8} & 6.2 {\tiny $\pm$ 0.94} \\
  ASAP-RL-PPO & RSS'23 & 2.86 {\tiny $\pm$ 1.11} & 0.04 {\tiny $\pm$ 0.002} & 3.99 {\tiny $\pm$ 0.37} & 2.25 {\tiny $\pm$ 0.07} & 0.21 {\tiny $\pm$ 0.13} & 9276 {\tiny $\pm$ 4064.5} & 394.4 {\tiny $\pm$ 201.2} & 26.1 {\tiny $\pm$ 19.5} \\
  ChatScene-SAC & CVPR'24 & 17.4 {\tiny $\pm$ 0.20} & 2.0 {\tiny $\pm$ 0.24}   & 717 {\tiny $\pm$ 68}    & 10.1 {\tiny $\pm$ 0.6} & 0.88 {\tiny $\pm$ 0.02}  & 1774 {\tiny $\pm$ 124}    & 4.6 {\tiny $\pm$ 0.5} & 1.63 {\tiny $\pm$ 0.13} \\
  ChatScene-PPO & CVPR'24 & 14.1 {\tiny $\pm$ 0.14} & 0.9 {\tiny $\pm$ 0.13}   & 248 {\tiny $\pm$ 44}    & 5.2 {\tiny $\pm$ 0.1}  & 0.83 {\tiny $\pm$ 0.02}  & 793 {\tiny $\pm$ 69}      & 14.1 {\tiny $\pm$ 0.4} & 3.62 {\tiny $\pm$ 0.02} \\
  \midrule
  
  \rowcolor{gray!15}
  \multicolumn{10}{l}{\textit{\textbf{LLM-based Reward Methods}}} \\
   Revolve & ICLR'25     & 17.6 {\tiny $\pm$ 0.71} & 1.9 {\tiny $\pm$ 1.07}   & 671  {\tiny $\pm$ 437 }  & 9.5 {\tiny $\pm$ 2.9} & 0.76 {\tiny $\pm$ 0.28}  & 1556 {\tiny $\pm$ 875}   & 20.5 {\tiny $\pm$ 25.1} & 5.48 {\tiny $\pm$ 6.19} \\
   Revolve-auto & ICLR'25 & 17.3 {\tiny $\pm$ 0.34} & 1.4 {\tiny $\pm$ 0.07}   & 485  {\tiny $\pm$ 11 }   & 6.0 {\tiny $\pm$ 2.4} & 0.83 {\tiny $\pm$ 0.23}  & 1390 {\tiny $\pm$ 403}   & 6.7 {\tiny $\pm$ 1.4} & 2.27 {\tiny $\pm$ 0.36} \\

  \midrule
  \rowcolor{gray!15}
  \multicolumn{10}{l}{\textit{\textbf{VLM-based Reward Methods}}} \\
   VLM-SR & NeurIPS'23 & 0.31 {\tiny $\pm$ 0.17} & 0.05 {\tiny $\pm$ 0.01} & 5.60 {\tiny $\pm$ 0.91} & 3.03 {\tiny $\pm$ 0.36} & 0.19 {\tiny $\pm$ 0.07} & 15096 {\tiny $\pm$ 4551.5} & 420.0 {\tiny $\pm$ 385.0} & 3.89 {\tiny $\pm$ 1.16} \\
   RoboCLIP & NeurIPS'23    & 0.47 {\tiny $\pm$ 0.23} & 0.22 {\tiny $\pm$ 0.11} & 34.0 {\tiny $\pm$ 21.2} & 4.04 {\tiny $\pm$ 0.59} & 0.29 {\tiny $\pm$ 0.05} & 12598 {\tiny $\pm$ 3002} & 85.3 {\tiny $\pm$ 46.1} & 2.00 {\tiny $\pm$ 0.54} \\
   VLM-RM & ICLR'24       & 0.15 {\tiny $\pm$ 0.09} & 0.07 {\tiny $\pm$ 0.05}   & 7.51 {\tiny $\pm$ 5.76 }   & 3.0 {\tiny $\pm$ 0.5} & 0.16 {\tiny $\pm$ 0.08}  & 17511 {\tiny $\pm$ 923}    & 918 {\tiny $\pm$ 77.6} & 3.53 {\tiny $\pm$ 2.11} \\
   LORD & Arxiv'24        & 0.08 {\tiny $\pm$ 0.06} & 0.04 {\tiny $\pm$ 0.02}   & 4.94 {\tiny $\pm$ 3.81 }   & 3.6 {\tiny $\pm$ 1.4} &  0.095 {\tiny $\pm$ 0.03}  & 16852 {\tiny $\pm$ 1826}   & 2904 {\tiny $\pm$ 1043} & 6.93 {\tiny $\pm$ 2.70} \\
   LORD-Speed & Arxiv'24 & 18.1 {\tiny $\pm$ 0.73} & 2.15 {\tiny $\pm$ 0.82} & 816.9 {\tiny $\pm$ 397.5} & 8.28 {\tiny $\pm$ 2.47} & 0.92 {\tiny $\pm$ 0.01} & 1954 {\tiny $\pm$ 918} & 5.4 {\tiny $\pm$ 3.8} & 1.93 {\tiny $\pm$ 1.32} \\   
   
  \rowcolor{green!10}
  VLM-RL (ours) & -      & 17.4 {\tiny $\pm$ 0.24} & \textbf{4.4} {\tiny $\pm$ 0.25}   & \textbf{1780} {\tiny $\pm$ 139 }  & \textbf{2.6} {\tiny $\pm$ 0.7} & 0.68 {\tiny $\pm$ 0.03}  & 5920 {\tiny $\pm$ 725}   & \textbf{2.5} {\tiny $\pm$ 0.7} & \textbf{0.76} {\tiny $\pm$ 0.14} \\

  \bottomrule
  \end{tabular}
  \end{small}
\end{table*}

\subsection{Main Results}

We present a detailed evaluation of our proposed VLM-RL against various baseline methods in Tabs. \ref{tab1}-\ref{tab2} and Figs. \ref{[main]-expert-binary}-\ref{[main]-vlm}. 
All experiments were conducted using three different random seeds and trained for 1 million steps to ensure statistical significance. 
The performance metrics reported in Tab. \ref{tab1} represent the mean and standard deviation of the final checkpoint during training across these three independent runs. 
For testing results, we selected the best-performing checkpoint from each training run based on comprehensive performance metrics, with the corresponding evaluation results presented in Tab. \ref{tab2}. 
The learning curves shown in Figs. \ref{[main]-expert-binary}-\ref{[main]-vlm} track the training progress of different methods, where solid lines indicate the mean performance across three seeds, and the shaded regions represent one standard deviation from the mean. 
This visualization allows us to observe not only the final performance but also the learning dynamics and stability of different approaches throughout the training process.

\subsubsection{Training Performance Analysis}

We first compare VLM-RL with expert-designed reward methods. From the training curves in Fig. \ref{[main]-expert-binary}, we can see that TIRL exhibits relatively low collision rates and high collision-free intervals. However, this seemingly positive performance is actually a result of the agent's failure to learn basic driving behaviors. 
As shown in Tab. \ref{tab1}, TIRL-SAC achieves only 0.01 km/h average speed, 0.01 route completion, and 0.21m total driving distance, indicating that the agent essentially remains stationary rather than learning to navigate.
In contrast, our VLM-RL demonstrates superior performance across all key metrics. It achieves an average speed of 17.4 km/h while maintaining a low collision speed of 2.6 km/h, and most importantly, successfully completes 4.4 routes with a total driving distance of 1780m. This comprehensive performance indicates that VLM-RL successfully learns both safe driving behaviors and effective navigation strategies.
The poor performance of TIRL can be attributed to the limitations of its simple binary reward design in the context of autonomous driving. Binary rewards that only penalize collisions (-1) while assigning neutral rewards (0) to all other states create a significant exploration challenge. 
In autonomous driving, where the action space is continuous and the state space is highly complex, such sparse binary rewards provide insufficient learning signals for the agent to discover productive driving behaviors. Without positive reinforcement for forward progress or successful navigation, the agent only learns to minimize collision risk by remaining stationary. 
This represents a local optimum that avoids negative rewards but fails to accomplish the actual driving objectives.

Compared to other expert-designed reward methods with weighted summation terms, VLM-RL demonstrates more balanced and stable performance. Chen-SAC achieves a higher average speed of 19.9 km/h but suffers from a high collision frequency of DCF and TCF and low route completion of 0.68 routes, suggesting that its reward design over-emphasizes speed at the expense of safety. 
ASAP-RL-PPO maintains relatively safe driving with a low collision rate of 0.21 but achieves limited progress with only 0.04 routes. This indicates its reward function may be too conservative, leading to overly cautious driving behaviors.
ChatScene variants show moderate performance across metrics but struggle to balance driving efficiency and safety, with their collision rates of 0.88 and 0.83 significantly higher than VLM-RL's 0.68. 
As shown in Fig. \ref{[main]-expert}, our method exhibits more stable and consistent learning progress. While other methods often show high variance or unstable behaviors during training, VLM-RL maintains steady improvement across all metrics. 
It is worth noting that the continuous decrease in collision rate during the latter half of training demonstrates the agent's improving ability to avoid dangerous situations while maintaining high driving efficiency. This learning pattern is unique to our method, as other approaches either show fluctuating collision rates or achieve safety at the cost of progress. 
Additionally, our method achieves the highest route completion of 4.4 and total driving distance of 1780m among all methods. 
This superior performance demonstrates that our CLG-based rewards and hierarchical reward synthesis provide more informative and balanced learning signals compared to manually designed reward functions, effectively guiding the agent to learn both safe and efficient driving behaviors.

\begin{figure*}
  \centerline{\includegraphics[width=0.993\textwidth]{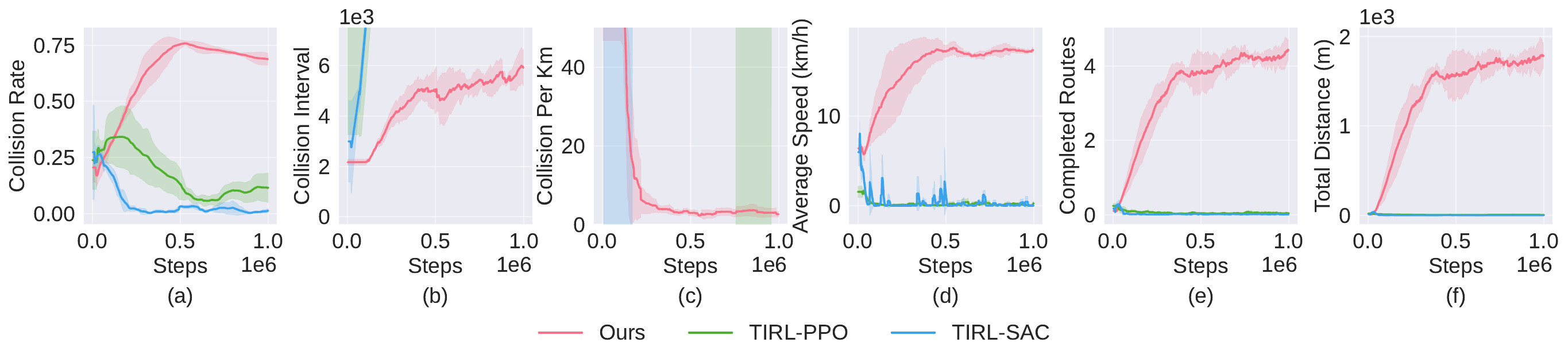}}
  \caption{Performance comparison with expert-designed baselines with binary rewards during training.}
  \label{[main]-expert-binary}
\end{figure*}

\begin{figure*}
  \centerline{\includegraphics[width=0.993\textwidth]{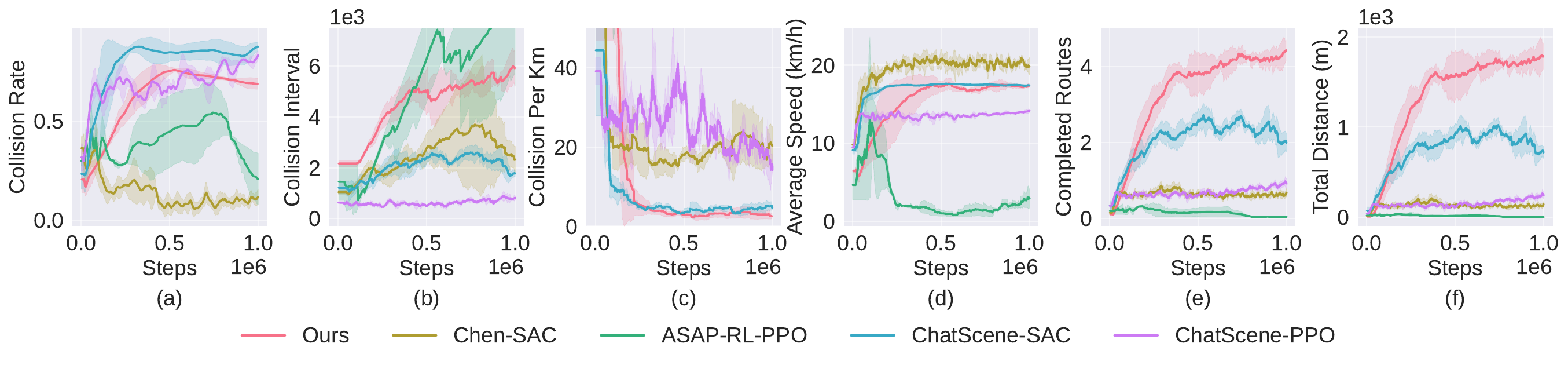}}
  \caption{Performance comparison with expert-designed baselines with summation rewards during training.}
  \label{[main]-expert}
\end{figure*}

Next, we compare our method with LLM-based reward approaches. As shown in Tab. \ref{tab1}, both Revolve and Revolve-auto achieve comparable driving speeds (17.6 km/h and 17.3 km/h) but exhibit relatively high collision rates (0.76 and 0.83) and limited route completion numbers (1.9 and 1.4). 
The training curves in Fig. \ref{[main]-llm} reveal interesting behavioral patterns. While Revolve variants quickly learn to achieve and maintain high average speeds comparable to VLM-RL, they struggle with safety aspects, as evidenced by their consistently high collision rates throughout training. 
In contrast, VLM-RL demonstrates a more balanced learning trajectory, gradually improving both driving efficiency and safety. Most notably, while Revolve and Revolve-auto show limited improvement in route completion and total distance traveled after the initial learning phase, VLM-RL continues to make steady progress, ultimately achieving more than twice the route completion rate (4.4 versus 1.9) and significantly longer driving distances (1780m versus 671m). 
This suggests that our CLG-based hierarchical reward design provides more comprehensive and well-structured learning signals compared to LLM-generated reward functions, enabling the agent to better balance the competing objectives of efficiency and safety in autonomous driving.

\begin{figure*}
  \centerline{\includegraphics[width=0.993\textwidth]{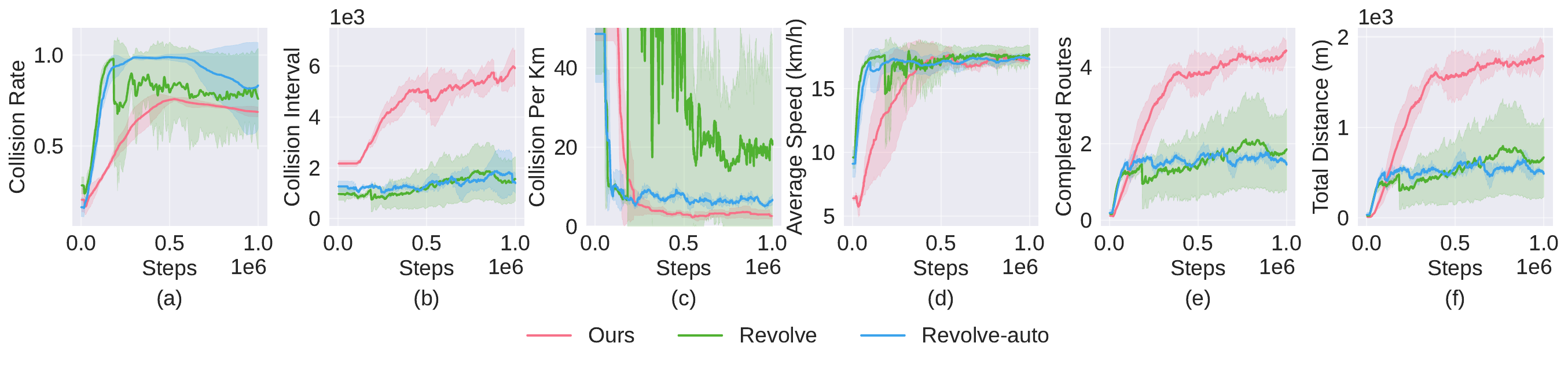}}
  \caption{Performance comparison with LLM-designed baselines during training.}
  \label{[main]-llm}
\end{figure*}

When comparing with VLM-based reward methods, we observe distinct performance patterns. 
VLM-SR, RoboCLIP, and VLM-RM, despite their demonstrated success in robotic tasks, show limited effectiveness in autonomous driving scenarios when relying solely on VLM-derived semantic rewards.
These methods achieve very low average speeds of 0.31, 0.47, and 0.15 km/h respectively, and route completion numbers all below 0.22, suggesting they struggle to learn basic driving behaviors. The training curves in Fig. \ref{[main]-vlm} (a)-(f) reveal that these methods maintain relatively low collision rates primarily because the agents remain nearly stationary, similar to the behavior observed with TIRL.
This performance gap between robotic manipulation and autonomous driving stems from fundamental task differences. 
Robotic tasks typically involve discrete, well-defined goal states such as grasping objects or manipulating tools that can be effectively captured by VLM-based similarity metrics. 
In contrast, autonomous driving requires continuous, dynamic decision-making where the desired behavior is a complex combination of multiple objectives that evolve over time. 
Our hierarchical reward synthesis approach addresses this challenge by combining high-level semantic understanding from CLG with low-level vehicle state information, providing comprehensive learning signals that better guide the driving policy.

\begin{figure*}
  \centerline{\includegraphics[width=0.993\textwidth]{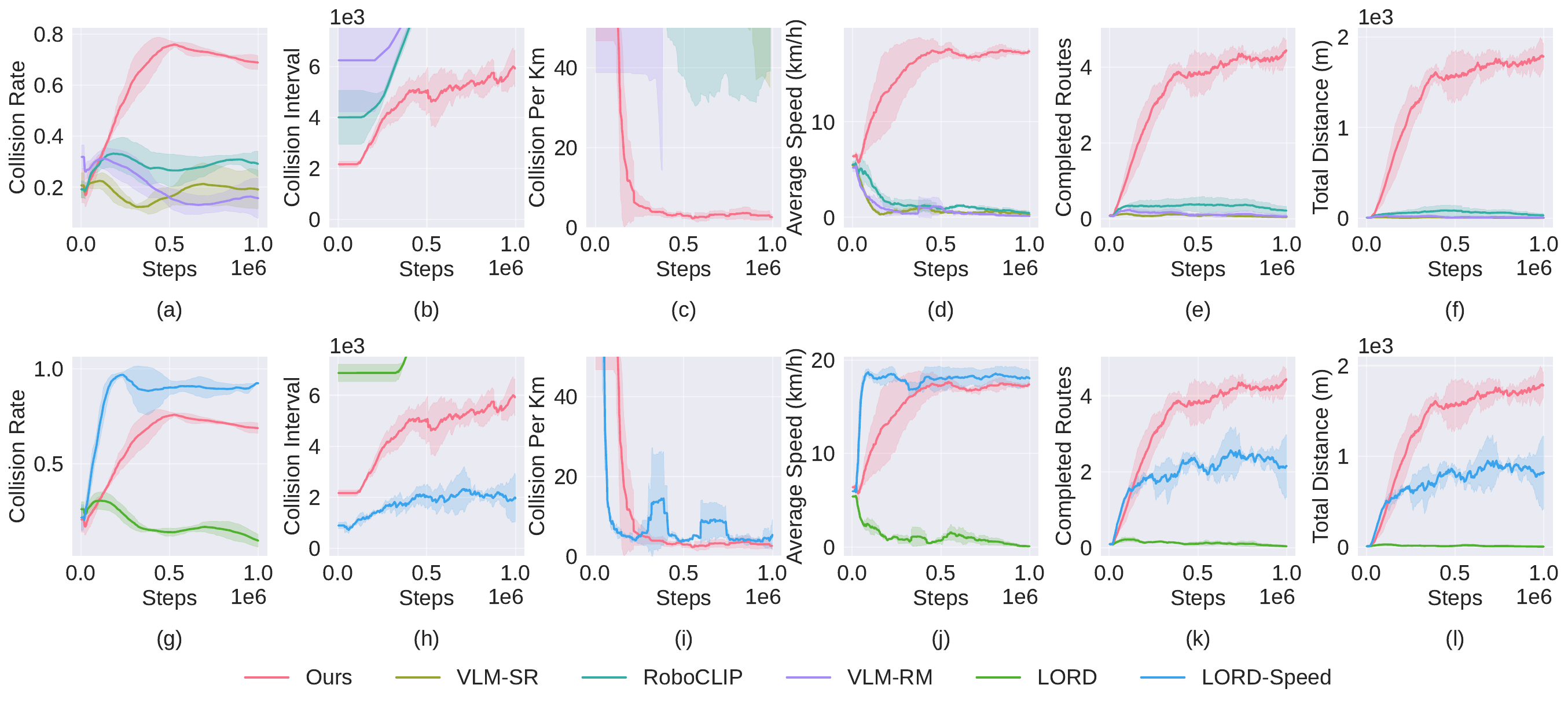}}
  \caption{Performance comparison with VLM-designed baselines during training.}
  \label{[main]-vlm}
\end{figure*}

LORD and its variant LORD-Speed, both designed specifically for autonomous driving, show contrasting performance patterns. While LORD achieves a low collision rate of 0.095, its average speed of 0.08 km/h and route completion of 0.04 indicate similar limitations as other VLM-only approaches. 
It is worth noting that LORD's original success was demonstrated in the HighwayEnv simulator \citep{highway-env} with discrete action spaces and simplified observation representation, which is considerably less complex than our CARLA environment with continuous control and realistic visual inputs. 
LORD-Speed, which incorporates additional speed-based rewards, shows significantly improved performance with an average speed of 18.1 km/h and a route completion rate of 2.15. However, this improvement comes at the cost of safety, as evidenced by its high collision rate of 0.92. 
As shown in Fig. \ref{[main]-vlm} (g)-(l), LORD-Speed quickly achieves high average speeds but fails to effectively balance safety and efficiency, maintaining consistently high collision rates throughout training. 
In contrast, VLM-RL demonstrates more balanced learning progress. This superior performance validates the effectiveness of our CLG-based and hierarchical reward design in providing more comprehensive and balanced learning signals compared to existing VLM-based approaches.

\begin{table}[!t]
\begin{small}
  \caption{Performance comparison with baselines during testing. Mean and standard deviation over 3 seeds. The best results are marked in \textbf{bold}.}
  \label{tab2}
  \centering
  \begin{center}
  \renewcommand{\arraystretch}{1.5}
  \begin{tabular}{@{}ccccccc@{}}
  \toprule
  Model & Reference & AS~$\uparrow$ & RC~$\uparrow$ & TD~$\uparrow$ & CS~$\downarrow$ & SR~$\uparrow$ \\
  \midrule 
  
  \rowcolor{gray!15}
  \multicolumn{7}{l}{\textit{\textbf{Expert-designed Reward Methods (Binary Rewards)}}} \\
  TIRL-SAC & TR-C'22  & 0.37 {\tiny $\pm$ 0.28} & 0.01 {\tiny $\pm$ 0.001} & 4.7 {\tiny $\pm$ 3.5} & 0.24 {\tiny $\pm$ 0.34} & 0.0 {\tiny $\pm$ 0.0}  \\
  TIRL-PPO & TR-C'22   & 0.43 {\tiny $\pm$ 0.23} & 0.01 {\tiny $\pm$ 0.005} & 14.8 {\tiny $\pm$ 9.7} &  0.10 {\tiny $\pm$ 0.15} & 0.0 {\tiny $\pm$ 0.0}  \\
  \midrule 
  
  \rowcolor{gray!15}
  \multicolumn{7}{l}{\textit{\textbf{Expert-designed Reward Methods (Summation Rewards)}}} \\
   Chen-SAC & T-ITS'22 & \textbf{21.4} {\tiny $\pm$ 1.16} & 0.29 {\tiny $\pm$ 0.12} & 663.6 {\tiny $\pm$ 286.7} & 2.07 {\tiny $\pm$ 2.21} & 0.08 {\tiny $\pm$ 0.08}  \\
   ASAP-RL-PPO  & RSS'23 & 1.25 {\tiny $\pm$ 0.30} & 0.01 {\tiny $\pm$ 0.00} & 28.1 {\tiny $\pm$ 3.44} & 0.61 {\tiny $\pm$ 0.61}  & 0.0 {\tiny $\pm$ 0.0}  \\
  ChatScene-SAC  & CVPR'24 & 17.7 {\tiny $\pm$ 0.12} & 0.88 {\tiny $\pm$ 0.03} & 1763.2 {\tiny $\pm$ 90.9} &  1.18 {\tiny $\pm$ 0.46} & 0.73 {\tiny $\pm$ 0.05}  \\
  ChatScene-PPO  & CVPR'24  & 15.3 {\tiny $\pm$ 0.33} & 0.78 {\tiny $\pm$ 0.05} & 1515.6 {\tiny $\pm$ 129.1} &  0.89 {\tiny $\pm$ 0.32} & 0.63 {\tiny $\pm$ 0.05}  \\
  \midrule

  \rowcolor{gray!15}
  \multicolumn{7}{l}{\textit{\textbf{LLM-based Reward Methods}}} \\
   Revolve  & ICLR'25  & 18.4 {\tiny $\pm$ 0.03} & 0.92 {\tiny $\pm$ 0.11} & 1915.3 {\tiny $\pm$ 248.3} &  1.53 {\tiny $\pm$ 2.16} & 0.83 {\tiny $\pm$ 0.24}  \\
   Revolve-auto  & ICLR'25   & 17.2 {\tiny $\pm$ 0.76} & 0.80 {\tiny $\pm$ 0.06} & 1539.6 {\tiny $\pm$ 147.5} &  1.65 {\tiny $\pm$ 0.28} & 0.63 {\tiny $\pm$ 0.05}  \\
  \midrule 
  
  \rowcolor{gray!15}
  \multicolumn{7}{l}{\textit{\textbf{VLM-based Reward Methods}}} \\
   VLM-SR  & NeurIPS'23 & 0.53 {\tiny $\pm$ 0.27} & 0.02 {\tiny $\pm$ 0.00} & 47.9 {\tiny $\pm$ 9.2} & 0.18 {\tiny $\pm$ 0.25} & 0.0 {\tiny $\pm$ 0.0}  \\
   RoboCLIP  & NeurIPS'23 & 0.44 {\tiny $\pm$ 0.05} & 0.07 {\tiny $\pm$ 0.03} & 146.3 {\tiny $\pm$ 62.3} & 1.05 {\tiny $\pm$ 0.58} &  0.0 {\tiny $\pm$ 0.0}  \\
   VLM-RM  & ICLR'24 & 0.20 {\tiny $\pm$ 0.05} & 0.02 {\tiny $\pm$ 0.01} & 35.9 {\tiny $\pm$ 25.8} & \textbf{0.003} {\tiny $\pm$ 0.005} & 0.0 {\tiny $\pm$ 0.0}  \\
   LORD   & Arxiv'24   & 0.17 {\tiny $\pm$ 0.08} & 0.02 {\tiny $\pm$ 0.02} & 45.1 {\tiny $\pm$ 57.1} &  0.02 {\tiny $\pm$ 0.02} & 0.0 {\tiny $\pm$ 0.0} \\
   LORD-Speed  & Arxiv'24 & 18.9 {\tiny $\pm$ 0.36} & 0.87 {\tiny $\pm$ 0.08} & 1783.4 {\tiny $\pm$ 172.8} & 2.80 {\tiny $\pm$ 1.16} & 0.67 {\tiny $\pm$ 0.05}  \\
   \rowcolor{green!10}
   VLM-RL (ours)  & -  & 19.3 {\tiny $\pm$ 1.29} & \textbf{0.97} {\tiny $\pm$ 0.03} & \textbf{2028.2} {\tiny $\pm$ 96.6} & 0.02 {\tiny $\pm$ 0.03} & \textbf{0.93} {\tiny $\pm$ 0.04}  \\
  \bottomrule
  \end{tabular}
  \end{center}
\end{small}
\end{table}

\subsubsection{Performance Evaluation in Testing}
To further validate the effectiveness of VLM-RL, we conduct comprehensive testing evaluations across 10 predefined routes and compare the performance with baseline methods. The route completion metric represents the average route completion rates during each evaluation episode. The testing results in Tab. \ref{tab2} demonstrate significant advantages of our approach compared to the baselines.

The limitations of binary reward methods remain evident in the testing phase. TIRL variants achieve a route completion rate of 0.01 and total driving distances of 4.7m and 14.8m respectively, confirming their failure to learn meaningful driving behaviors. 
Among expert-designed reward methods with weighted summation terms, Chen-SAC maintains the highest average speed at 21.4 km/h but shows limited effectiveness with a 0.08 success rate and 0.29 route completion, indicating its aggressive driving style compromises mission success. 
ChatScene variants demonstrate more balanced performance with success rates of 0.73 and 0.63 respectively, though their collision speeds of 1.18 km/h and 0.89 km/h suggest potential safety concerns.

LLM-based approaches demonstrate competitive performance during testing, with Revolve achieving a success rate of 0.83 and route completion of 0.92. However, their collision speeds of 1.53 km/h and 1.65 km/h indicate persistent safety issues. Most VLM-based methods, including VLM-SR, RoboCLIP, VLM-RM, and LORD, exhibit highly conservative behaviors with route completion rates below 0.07 and success rates of 0.0. LORD-Speed shows significantly improved efficiency metrics but records the highest collision speed at 2.80 km/h among all methods.

In contrast, VLM-RL achieves superior performance across all key metrics during testing. It maintains a high average speed of 19.3 km/h while recording a low collision speed of 0.02 km/h, matching the safety level of the most conservative approaches. Most notably, VLM-RL achieves the highest success rate of 0.93 and route completion of 0.97, along with the longest total driving distance of 2028.2m. These results demonstrate that our method not only learns more effective driving policies but also exhibits better generalization to testing scenarios. The significant improvements in both efficiency and safety metrics validate the effectiveness of our CLG-based and hierarchical reward design in providing comprehensive and well-balanced learning signals for safe driving tasks.

\subsection{Ablation Study}

Building upon our previous baseline comparisons with VLM-SR, RoboCLIP, VLM-RM and LORD, which established the advantages of our hierarchical reward synthesis approach, we conduct ablation studies to further validate the effectiveness of our proposed CLG approach. 
Specifically, we investigate the performance when using only positive language goals  (VLM-RL-pos) and only negative language goals (VLM-RL-neg), respectively. 
These variants allow us to analyze the individual contribution of each goal type and demonstrate why combining both through our contrasting framework leads to superior performance. 
Additionally, we compare the performance of using CARLA's built-in segmentation camera-based BEV as the RL agent's observation (VLM-RL-bev) as an ablation experiment to validate the effectiveness of the BEV design shown in Fig. \ref{bev}. These ablation experiments provide additional insights into the specific mechanisms that contribute to our method's effectiveness.

\begin{figure*}
  \centerline{\includegraphics[width=0.993\textwidth]{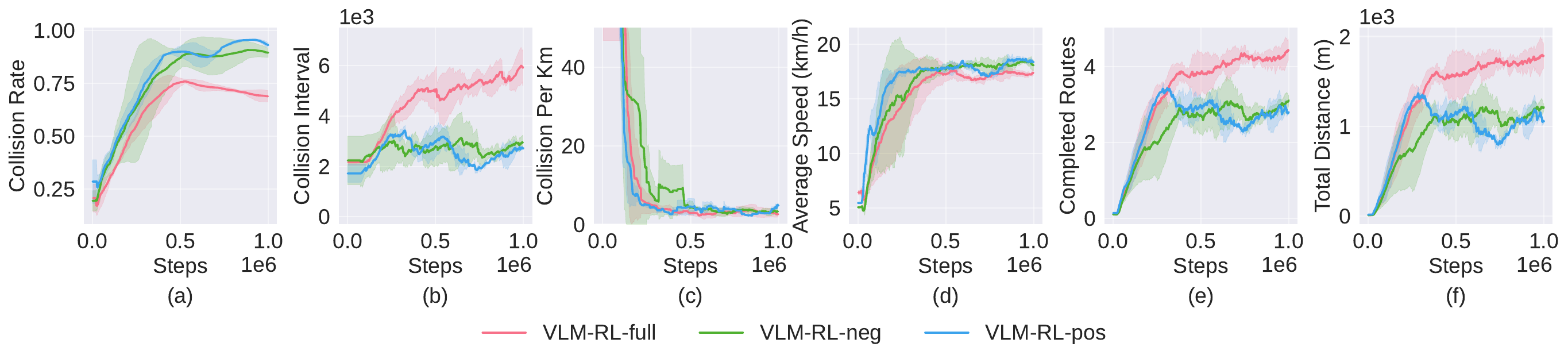}}
  \caption{Performance comparison using different language goals during training.}
  \label{ablation_analysis_pos_neg}
\end{figure*}

\begin{table}
  \caption{Ablation study on the model components during testing. Mean and standard deviation over 3 seeds. The best results are marked in \textbf{bold}.}
  \label{tab3}
  \centering
  \begin{center}
  \renewcommand{\arraystretch}{1.5}
  \begin{tabular}{@{}cccccc@{}}
  \toprule
  Model & AS~$\uparrow$ & RC~$\uparrow$ & TD~$\uparrow$ & CS~$\downarrow$ & SR~$\uparrow$  \\
  \midrule 
  
   VLM-RL-bev  & 18.5 {\tiny $\pm$ 1.58} & 0.92 {\tiny $\pm$ 0.05} & 1905.5 {\tiny $\pm$ 129.7} & 0.48 {\tiny $\pm$ 0.51} & 0.78 {\tiny $\pm$ 0.07}  \\
   VLM-RL-neg    & 19.2 {\tiny $\pm$ 0.72} & 0.90 {\tiny $\pm$ 0.08} & 1901.5 {\tiny $\pm$ 228.1} & 0.94 {\tiny $\pm$ 0.95} & 0.77 {\tiny $\pm$ 0.09}  \\
   VLM-RL-pos & \textbf{19.4} {\tiny $\pm$ 0.88} & 0.89 {\tiny $\pm$ 0.10} & 1817.3 {\tiny $\pm$ 263.3} & 1.47 {\tiny $\pm$ 1.34} & 0.75 {\tiny $\pm$ 0.05}  \\
   \rowcolor{green!10}
   VLM-RL-full  & 19.3 {\tiny $\pm$ 1.29} & \textbf{0.97} {\tiny $\pm$ 0.03} & \textbf{2028.2} {\tiny $\pm$ 96.6} & \textbf{0.02} {\tiny $\pm$ 0.03} & \textbf{0.93} {\tiny $\pm$ 0.04}  \\
  \bottomrule
  \end{tabular}
  \end{center}
\end{table}

As shown in Fig. \ref{ablation_analysis_pos_neg}, we can observe clear patterns that demonstrate the advantages of combining both positive and negative language goals in our full VLM-RL model compared to its variants. In terms of safety metrics, the collision rate shows that both VLM-RL-pos and VLM-RL-neg tend to converge to a higher collision rate. 
The collision interval further supports this observation, with VLM-RL-full maintaining significantly longer intervals between collisions, reaching nearly 6000 steps compared to around 2500-3000 steps for the variants. 
Regarding driving efficiency, the average speed indicates that all three models eventually achieve comparable speeds of around 17-18 km/h. The most striking differences appear in the completed routes and total distance traveled, where VLM-RL-full significantly outperforms both variants, completing 4.4 routes compared to about 3 routes for the variants, and covering nearly 1800m versus approximately 1200m.

The testing results in Tab. \ref{tab3} further validate these observations and provide insights into the effectiveness of our BEV design. 
While VLM-RL-bev achieves competitive performance in terms of average speed (18.5 km/h) and route completion (0.92), its collision speed of 0.48 km/h and success rate of 0.78 indicate compromised safety compared to our full model. This suggests that our custom BEV design better captures critical environmental features for safe navigation. 
The single-goal variants (VLM-RL-neg and VLM-RL-pos) show similar patterns during testing, with high average speeds but elevated collision speeds of 0.94 km/h and 1.47 km/h respectively. 
In contrast, VLM-RL-full maintains comparable efficiency while achieving a remarkably low collision speed of 0.02 km/h and the highest success rate of 0.93. These results demonstrate that the CLG approach leads to more balanced and effective learning, enabling the agent to better navigate the trade-off between safety and efficiency in autonomous driving.

\subsection{VLM-RL Performance Scaling Across CLIP Model Sizes}

To systematically investigate how the scale of VLM affects the performance of VLM-RL, we conducted experiments with four different CLIP model variants of increasing size and complexity: ViT-B-32 (base), ViT-L-14-quickgelu (large), ViT-H-14 (huge), and ViT-bigG-14 (giant).
These models exhibit significant differences in their architectural parameters, ranging from 86M parameters in the baseline ViT-B-32 to over 1B parameters in the ViT-bigG-14. Additionally, they utilize different vision encoder configurations: while ViT-B-32 processes images using 32$\times$32 patches, the larger models employ finer 14$\times$14 patch sizes for increased granularity in visual feature extraction.

\begin{figure*}
  \centerline{\includegraphics[width=0.993\textwidth]{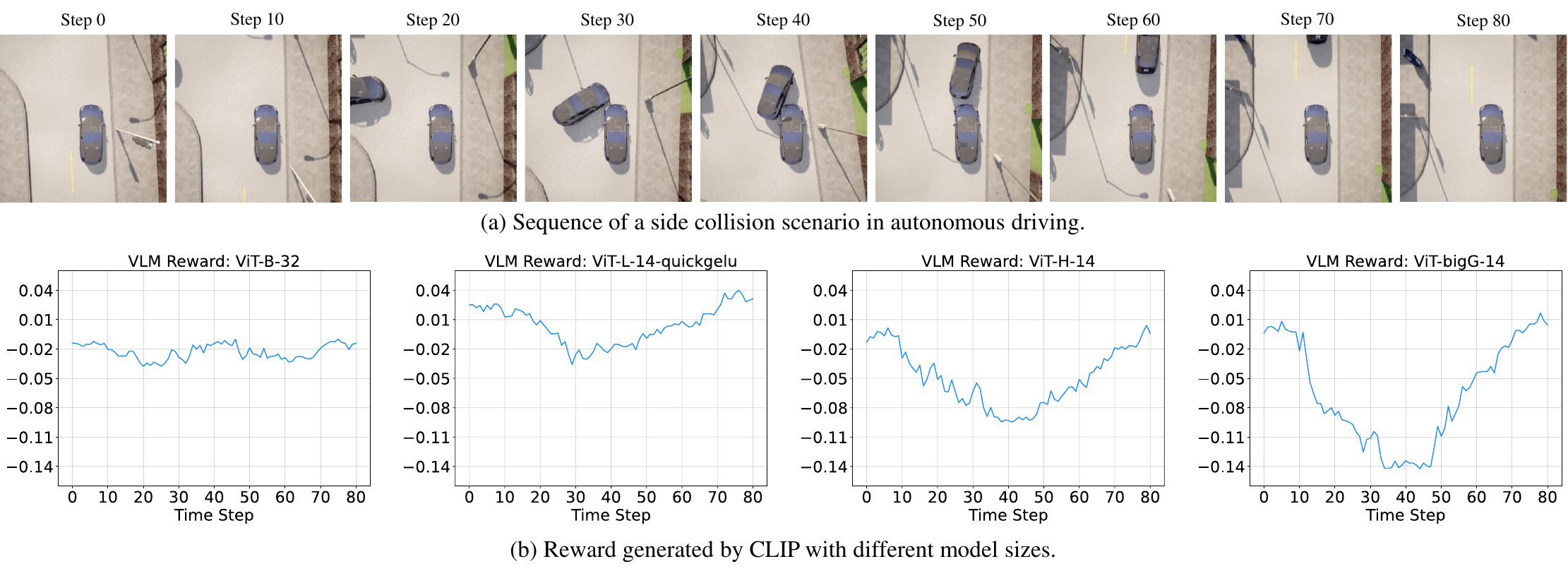}}
  \caption{Base rewards with different CLIP model sizes using the CLG in a side collision scenario.}
  \label{clip_size_analysis}
\end{figure*}

As shown in Fig. \ref{clip_size_analysis}, we first analyze the semantic reward signals generated by different CLIP model variants through a representative case study. The sequence of BEVs in Fig. \ref{clip_size_analysis} (a) depicts a critical safety situation where the agent vehicle experiences a side collision with another vehicle. 
The reward curves in Fig. \ref{clip_size_analysis} (b) reveal a clear correlation between model size and reward signal quality. The smallest model, ViT-B-32, produces relatively flat rewards fluctuating around -0.02, indicating poor sensitivity to the collision event. 
In comparison, ViT-bigG-14 demonstrates remarkably superior performance among all variants, generating the most distinctive and interpretable reward pattern with a sharp decline to -0.14 during the collision phase from steps 30 to 50 before gradually recovering as the vehicles separate. 
The significant enhancement in semantic reward signal discrimination observed in larger models, particularly exemplified by ViT-bigG-14, indicates that increased model capacity enables more sophisticated scene understanding capabilities and consequently generates more effective learning signals for reinforcement learning agents.

\begin{figure*}
  \centerline{\includegraphics[width=0.993\textwidth]{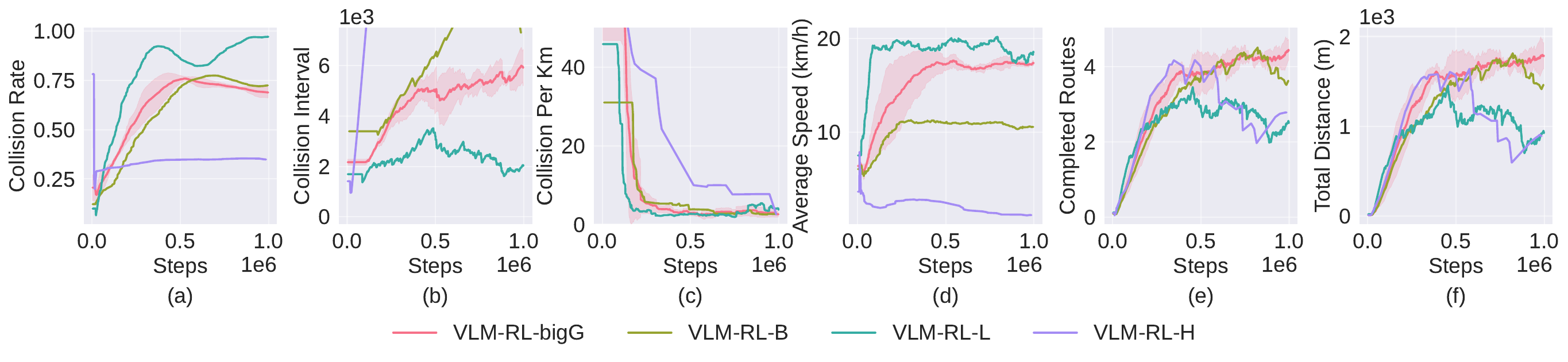}}
  \caption{Performance comparison using different CLIP model sizes during training.}
  \label{ablation_analysis_clip}
\end{figure*}

We further evaluate how different CLIP model sizes affect the performance of RL agents. Fig. \ref{ablation_analysis_clip} presents a comprehensive comparison for four variants: VLM-RL-B, VLM-RL-L, VLM-RL-H, and VLM-RL-bigG. 
The results reveal interesting trade-offs between safety and efficiency across model scales. 
While VLM-RL-H achieves the lowest collision rate, its low average speed of around 2 km/h suggests that the agent fails to learn basic navigation skills, essentially prioritizing safety at the complete expense of functionality.
Conversely, VLM-RL-L exhibits the highest average speed among all variants but at the cost of the highest collision rate, indicating an unsafe bias toward aggressive driving.
VLM-RL-B shows comparable performance to VLM-RL-bigG in terms of completed routes and total distance traveled, but operates at significantly lower speeds, suggesting less efficient navigation.
Among all variants, VLM-RL-bigG demonstrates the most balanced performance, maintaining a moderate collision rate while achieving efficient navigation as evidenced by its competitive speed, route completion, and travel distance metrics. 
These results indicate that the increased model capacity of ViT-bigG-14 enables the agent to better balance the complex trade-offs between safety and efficiency in autonomous driving.

\subsection{Generalization}

Generalization capability is crucial for autonomous driving systems, as they must adapt to diverse environments and conditions beyond their training scenarios. To comprehensively evaluate VLM-RL's adaptability to diverse scenarios, we first compare our method with two best-performing baseline approaches, i.e., ChatScene-SAC and Revolve, across different towns and varying traffic densities. We also demonstrate the versatility of our approach by extending it to PPO algorithm and comparing with PPO-based baselines, showing that our CLG-based reward design is effective across different RL frameworks.

\subsubsection{Different Towns}\label{sec5.7.1}

\begin{table}
\begin{small}
  \caption{Generalization performance across different Towns. Mean and standard deviation over 3 seeds. The best results are marked in \textbf{bold}.}
  \label{tab4}
  \centering
  \begin{center}
  \renewcommand{\arraystretch}{1.5}
  \begin{tabular}{@{}cccccccc@{}}
  \toprule
  Town & Model & AS~$\uparrow$ & RC~$\uparrow$ & TD~$\uparrow$ & CS~$\downarrow$ & SR~$\uparrow$  \\
  \midrule 
  
   \multirow{3}{*}{Town 1} 
   & ChatScene-SAC & 18.2 {\tiny $\pm$ 0.25} & 0.83 {\tiny $\pm$ 0.09} & 4408.8 {\tiny $\pm$ 671.0} & 4.50 {\tiny $\pm$ 2.84} & 0.67 {\tiny $\pm$ 0.17} \\
   & Revolve & 18.9 {\tiny $\pm$ 0.28} & 0.84 {\tiny $\pm$ 0.07} & 4458.7 {\tiny $\pm$ 449.5} & 4.27 {\tiny $\pm$ 2.00} & 0.72 {\tiny $\pm$ 0.17} \\
    & VLM-RL  & \textbf{22.9} {\tiny $\pm$ 0.63} & \textbf{1.00} {\tiny $\pm$ 0.00} & \textbf{5697.6} {\tiny $\pm$ 0.79} & \textbf{0.03} {\tiny $\pm$ 0.05} & \textbf{1.00} {\tiny $\pm$ 0.00}  \\  
  \midrule 
  
   \multirow{3}{*}{Town 2}  
    & ChatScene-SAC & 17.7 {\tiny $\pm$ 0.12} & 0.88 {\tiny $\pm$ 0.03} & 1763.2 {\tiny $\pm$ 90.9} &  1.18 {\tiny $\pm$ 0.46} & 0.73 {\tiny $\pm$ 0.05}  \\
    & Revolve & 18.4 {\tiny $\pm$ 0.03} & 0.92 {\tiny $\pm$ 0.11} & 1915.3 {\tiny $\pm$ 248.3} &  1.53 {\tiny $\pm$ 2.16} & 0.83 {\tiny $\pm$ 0.24}  \\
    & VLM-RL &  \textbf{19.3} {\tiny $\pm$ 1.29} & \textbf{0.97} {\tiny $\pm$ 0.03} & \textbf{2028.2} {\tiny $\pm$ 96.6} & \textbf{0.02} {\tiny $\pm$ 0.03} & \textbf{0.93} {\tiny $\pm$ 0.04}  \\
  \midrule 
  
   \multirow{3}{*}{Town 3} 
   & ChatScene-SAC & 17.8 {\tiny $\pm$ 0.22} & 0.85 {\tiny $\pm$ 0.04} & 3436.5 {\tiny $\pm$ 355.4} & 1.60 {\tiny $\pm$ 0.44}& 0.77 {\tiny $\pm$ 0.12}\\
   & Revolve & 18.5 {\tiny $\pm$ 0.41} & 0.75 {\tiny $\pm$ 0.16} & 2979.3 {\tiny $\pm$ 649.5} & 5.22 {\tiny $\pm$ 0.77} & 0.53 {\tiny $\pm$ 0.21}\\
   & VLM-RL &  \textbf{21.7} {\tiny $\pm$ 0.55} & \textbf{0.91} {\tiny $\pm$ 0.07} & \textbf{3757.8} {\tiny $\pm$ 248.0} & \textbf{1.14} {\tiny $\pm$ 1.54} & \textbf{0.87} {\tiny $\pm$ 0.09}  \\
  \midrule 
  
   \multirow{3}{*}{Town 4} 
   & ChatScene-SAC & 18.1 {\tiny $\pm$ 0.32} & \textbf{0.82} {\tiny $\pm$ 0.08} & \textbf{14139.0} {\tiny $\pm$ 1650.9} & 6.85 {\tiny $\pm$ 1.92} & \textbf{0.70} {\tiny $\pm$ 0.08}\\
   & Revolve & 17.0 {\tiny $\pm$ 2.88} & 0.63 {\tiny $\pm$ 0.10}  & 9874.4 {\tiny $\pm$ 2132.1} & 10.0 {\tiny $\pm$ 1.87} & 0.40 {\tiny $\pm$ 0.08}\\
   &  VLM-RL & \textbf{22.0} {\tiny $\pm$ 3.69} & 0.80 {\tiny $\pm$ 0.17} & 12684.1{\tiny $\pm$ 3608.5} & \textbf{2.15} {\tiny $\pm$ 1.59} & 0.70 {\tiny $\pm$ 0.28}  \\
  \midrule 
  
   \multirow{3}{*}{Town 5} 
   & ChatScene-SAC & 18.4 {\tiny $\pm$ 0.06} & 0.85 {\tiny $\pm$ 0.04} & 2826.3 {\tiny $\pm$ 350.6} & 2.14 {\tiny $\pm$ 1.34} & 0.73 {\tiny $\pm$ 0.09} \\
   & Revolve & 19.0 {\tiny $\pm$ 0.63} & 0.78 {\tiny $\pm$ 0.06} & 2671.6 {\tiny $\pm$ 231.4} & 5.35 {\tiny $\pm$ 1.95} & 0.53 {\tiny $\pm$ 0.09} \\
   & VLM-RL  & \textbf{22.9} {\tiny $\pm$ 0.86} & \textbf{0.93} {\tiny $\pm$ 0.03} & \textbf{3322.5} {\tiny $\pm$ 203.4} & \textbf{0.46} {\tiny $\pm$ 0.54} & \textbf{0.87} {\tiny $\pm$ 0.05}  \\
   \bottomrule
  \end{tabular}
  \end{center}
\end{small}
\end{table}

To evaluate the generalization capability of VLM-RL across different urban and rural environments, we test our model in five distinct towns within the CARLA simulator. As shown in Fig. \ref{towns}, Town 2 serves as the training environment, while Towns 1, 3, 4, and 5 represent previously unseen environments with varying layouts and road structures. The detailed performance comparison is shown in Tab. \ref{tab4}.

In Town 1, VLM-RL demonstrates exceptional performance with perfect success and route completion rates of 1.00, while achieving the highest average speed of 22.9 km/h and maintaining a remarkably low collision speed of 0.03 km/h. In contrast, both baseline methods show reduced performance with success rates below 0.72 and substantially higher collision speeds above 4.20 km/h.
The performance advantage persists in the more challenging Towns 3 and 5, where VLM-RL maintains high success rates of 0.87 and route completion rates above 0.91. Notably, VLM-RL achieves significantly lower collision speeds compared to the baselines, particularly evident in Town 5 where Revolve exhibits a high collision speed of 5.35 km/h versus VLM-RL's 0.46 km/h.
Town 4 presents the most challenging scenario with longer routes, as reflected in the total driving distances exceeding 12000m. In this environment, while ChatScene-SAC achieves marginally better route completion and success rates, VLM-RL maintains superior driving efficiency with the highest average speed of 22.0 km/h and significantly better safety performance, recording a collision speed of 2.15 km/h compared to 6.85 km/h and 10.0 km/h of the baselines.

These results demonstrate that VLM-RL successfully generalizes to diverse driving environments without additional training. The consistent performance advantages across different towns validate that our CLG-based and hierarchical reward design captures fundamental driving principles rather than overfitting to specific environmental features. This robust generalization capability is particularly crucial for real-world autonomous driving applications where vehicles must navigate diverse and previously unseen environments.

\subsubsection{Different Traffic Densities}

\begin{table}
\begin{small}
  \caption{Generalization performance across different traffic densities. Mean and standard deviation over 3 seeds. The best results are marked in \textbf{bold}.}
  \label{tab5}
  \centering
  \begin{center}
  \renewcommand{\arraystretch}{1.5}
  \begin{tabular}{@{}cccccccc@{}}
  \toprule
  Traffic Density & Model &  AS~$\uparrow$ & RC~$\uparrow$ & TD~$\uparrow$ & CS~$\downarrow$ & SR~$\uparrow$  \\
  \midrule 
  
   \multirow{3}{*}{Empty} 
   & ChatScene-SAC & 18.0 {\tiny $\pm$ 0.18} & 1.0 {\tiny $\pm$ 0.0} & 2064.1 {\tiny $\pm$ 4.81} & 0.0 {\tiny $\pm$ 0.0} & 1.0 {\tiny $\pm$ 0.0} \\
   & Revolve & 18.6 {\tiny $\pm$ 0.39} & 1.0 {\tiny $\pm$ 0.0} & 2105.5 {\tiny $\pm$ 21.0} & 0.0 {\tiny $\pm$ 0.0} & 1.0 {\tiny $\pm$ 0.0} \\
   & VLM-RL &  \textbf{23.8} {\tiny $\pm$ 0.29} & \textbf{1.0} {\tiny $\pm$ 0.0} & \textbf{2113.9} {\tiny $\pm$ 0.67} & \textbf{0.0} {\tiny $\pm$ 0.0} & \textbf{1.0} {\tiny $\pm$ 0.0}  \\
  \midrule 
  
   \multirow{3}{*}{Regular} 
    & ChatScene-SAC & 17.7 {\tiny $\pm$ 0.12} & 0.88 {\tiny $\pm$ 0.03} & 1763.2 {\tiny $\pm$ 90.9} &  1.18 {\tiny $\pm$ 0.46} & 0.73 {\tiny $\pm$ 0.05}  \\
    & Revolve & 18.4 {\tiny $\pm$ 0.03} & 0.92 {\tiny $\pm$ 0.11} & 1915.3 {\tiny $\pm$ 248.3} &  1.53 {\tiny $\pm$ 2.16} & 0.83 {\tiny $\pm$ 0.24}  \\
   & VLM-RL & \textbf{19.3} {\tiny $\pm$ 1.29} & \textbf{0.97} {\tiny $\pm$ 0.03} & \textbf{2028.2} {\tiny $\pm$ 96.6} & \textbf{0.02} {\tiny $\pm$ 0.03} & \textbf{0.93} {\tiny $\pm$ 0.04} \\
  \midrule 
  
   \multirow{3}{*}{Dense} 
   & ChatScene-SAC & 17.3 {\tiny $\pm$ 0.15} & 0.85 {\tiny $\pm$ 0.06} & 1734.2 {\tiny $\pm$ 92.3} & 2.71 {\tiny $\pm$ 0.65} & 0.77 {\tiny $\pm$ 0.12} \\
   & Revolve & \textbf{18.3} {\tiny $\pm$ 0.31} & \textbf{0.89} {\tiny $\pm$ 0.06} & \textbf{1861.0} {\tiny $\pm$ 144.3} & 3.53 {\tiny $\pm$ 0.52} & 0.73 {\tiny $\pm$ 0.05} \\
   & VLM-RL & 16.1 {\tiny $\pm$ 1.00} & 0.87 {\tiny $\pm$ 0.06} & 1819.0 {\tiny $\pm$ 166.6} & \textbf{0.11} {\tiny $\pm$ 0.10} & \textbf{0.80} {\tiny $\pm$ 0.08}  \\
  \bottomrule
  \end{tabular}
  \end{center}
\end{small}
\end{table}

To evaluate the robustness of VLM-RL under varying traffic conditions, we test our model in three traffic density settings as shown in Tab. \ref{tab5}: empty scenarios with no other vehicles, regular scenarios which is our default setting with 20 autopilot vehicles, and dense scenarios with 40 autopilot vehicles.

In empty scenarios, all methods achieve perfect success and route completion rates of 1.0, with zero collision speeds, demonstrating their fundamental capability in obstacle-free environments. However, VLM-RL exhibits superior driving efficiency with an average speed of 23.8 km/h, significantly higher than ChatScene-SAC at 18.0 km/h and Revolve at 18.6 km/h.
In dense traffic scenarios, the performance differences become more pronounced in terms of safety. While Revolve achieves marginally better efficiency metrics with the highest average speed of 18.3 km/h and route completion of 0.89, its collision speed increases significantly to 3.53 km/h. Similarly, ChatScene-SAC shows degraded safety performance with a collision speed of 2.71 km/h. In contrast, VLM-RL maintains excellent safety with a collision speed of 0.11 km/h while achieving comparable route completion of 0.87 and the highest success rate of 0.80.
These results demonstrate an important characteristic of VLM-RL. As traffic density increases, the model adapts by prioritizing safety over speed, demonstrating intelligent risk-aware behavior. This adaptive balance between efficiency and safety is particularly valuable for real-world autonomous driving systems that must operate safely across diverse traffic conditions.

\subsubsection{Different RL Algorithms}

\begin{figure*}[!t]
  \centerline{\includegraphics[width=0.993\textwidth]{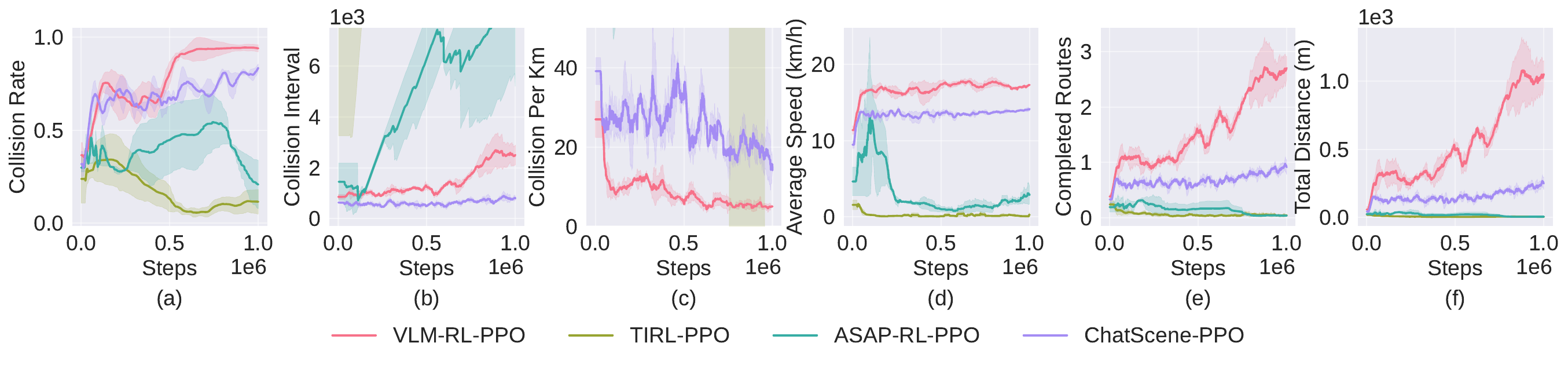}}
  \caption{Performance comparison among PPO-based methods.}
  \label{ppo}
\end{figure*}

We extend our VLM-RL framework to the PPO algorithm to demonstrate its compatibility with different RL algorithms. Fig. \ref{ppo} compares the training performance of VLM-RL-PPO with other PPO-based baselines including TIRL-PPO, ASAP-RL-PPO, and ChatScene-PPO. The results show that our CLG-based hierarchical reward design maintains its effectiveness when implemented with PPO.

As illustrated in Fig. \ref{ppo}, VLM-RL-PPO demonstrates superior performance across most metrics. While TIRL-PPO shows low collision rates, its average speed and completed routes remain close to zero throughout training, indicating the agent fails to learn meaningful driving behaviors. 
ASAP-RL-PPO achieves more stable training but shows limited progress in route completion and total distance traveled. 
ChatScene-PPO exhibits better driving capabilities with moderate average speed but struggles to maintain consistent performance, as shown by the fluctuating collision rates and limited route completion.
Most notably, VLM-RL-PPO shows remarkable learning progress in navigation capabilities. The number of completed routes demonstrates a strong upward trend throughout training, reaching approximately 2.5 routes compared to less than 1 route for other methods. Similarly, the total distance traveled exhibits substantial and consistent growth, ultimately achieving around 1000m, while other PPO-based methods remain below 300m. 
This improvement in both metrics indicates that VLM-RL-PPO effectively learns to navigate complex environments and complete driving tasks. Meanwhile, it maintains a stable average speed of around 15 km/h while successfully managing collision risks.
These results validate that our CLG-based hierarchical reward design can be effectively integrated with different RL algorithms, suggesting the broader applicability of our approach in safe driving tasks.

\subsection{Visualization of VLM Semantic Rewards}

\begin{figure*}
  \centerline{\includegraphics[width=0.993\textwidth]{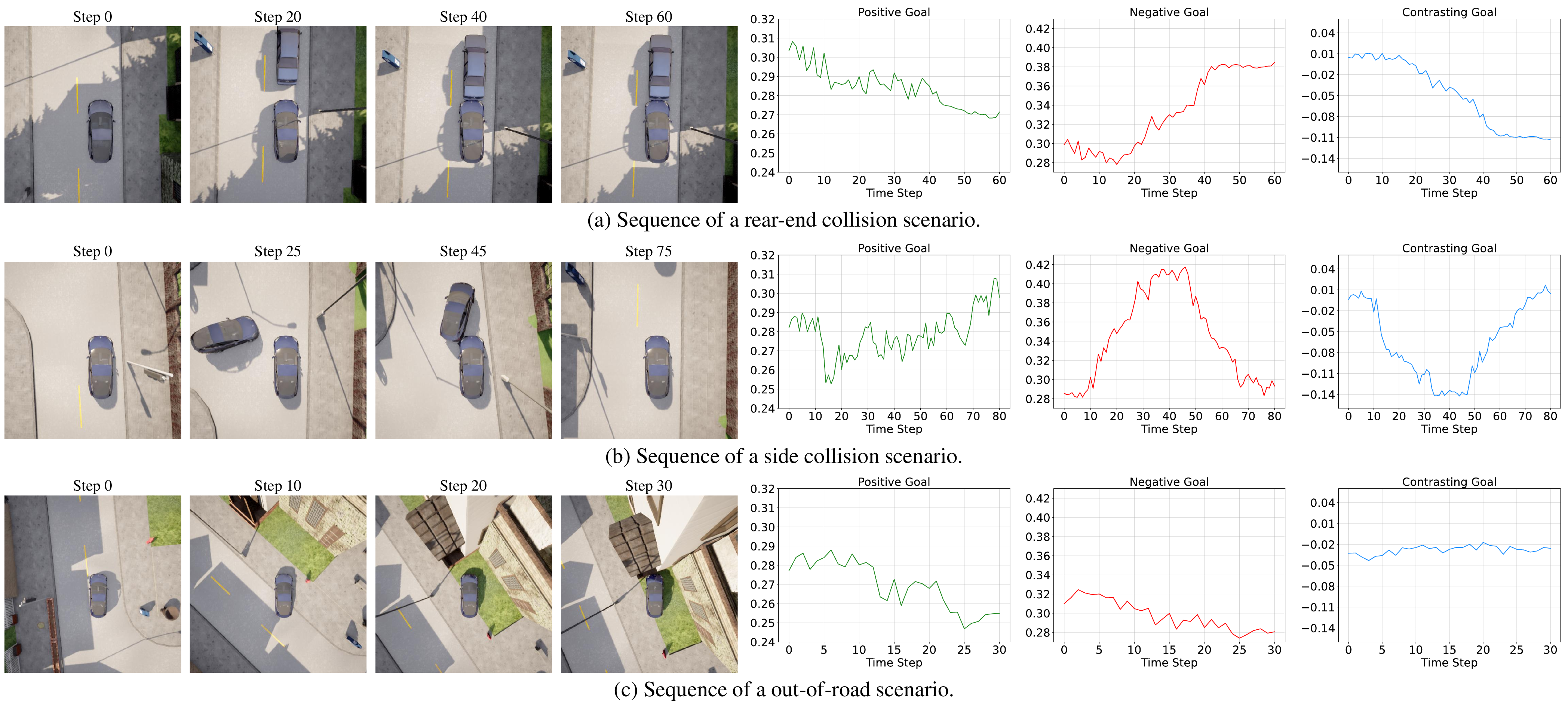}}
  \caption{CLIP rewards using different CLGs in different scenarios.}
  \label{prompts_analysis}
\end{figure*}

To better understand the effectiveness and limitations of CLIP-derived semantic rewards, we visualize three representative failure cases and analyze how the positive, negative, and CLG capture different driving behaviors. Fig. \ref{prompts_analysis} shows image sequences and their corresponding semantic similarity scores for three challenging scenarios.

In the rear-end collision scenario shown in Fig. \ref{prompts_analysis} (a), we observe that the positive goal similarity score gradually decreases as the ego vehicle approaches the leading vehicle, while the negative goal similarity increases. This pattern aligns with our design intention, where unsafe behaviors should result in lower positive goal similarity and higher negative goal similarity. The final contrasting score shows a clear downward trend, correctly reflecting the undesirable nature of the rear-end collision behavior.
The side collision case illustrated in Fig. \ref{prompts_analysis} (b) reveals another interesting pattern. The positive goal similarity fails to show a significant decrease during the collision period from steps 30 to 45, suggesting that VLM sometimes struggles to capture safety violations from the positive perspective alone. However, the negative goal similarity exhibits a pronounced peak during the collision, leading to an appropriate dip in the contrasting score that properly penalizes this unsafe behavior.
The out-of-road scenario presented in Fig. \ref{prompts_analysis} (c) highlights a limitation of purely VLM-based semantic rewards. While the positive goal similarity decreases as the vehicle deviates from the road, the negative goal similarity also shows a declining trend. This results in a relatively flat contrasting score that fails to adequately penalize the out-of-road behavior. This limitation stems from our language goals primarily focusing on collision-related behaviors, lacking explicit consideration of road boundary violations.

These observations highlight both the potential and limitations of VLM-derived semantic rewards. While the CLG can effectively capture many unsafe behaviors, it may fail to provide appropriate learning signals for certain scenarios, particularly those not explicitly described in the language goals. This underscores the necessity of our hierarchical reward synthesis approach, which combines this high-level semantic understanding with low-level vehicle state information to provide more comprehensive and reliable reward signals.

\subsection{Analysis of Hierarchical Reward Synthesis}
\begin{figure*}[!t]
  \centerline{\includegraphics[width=0.993\textwidth]{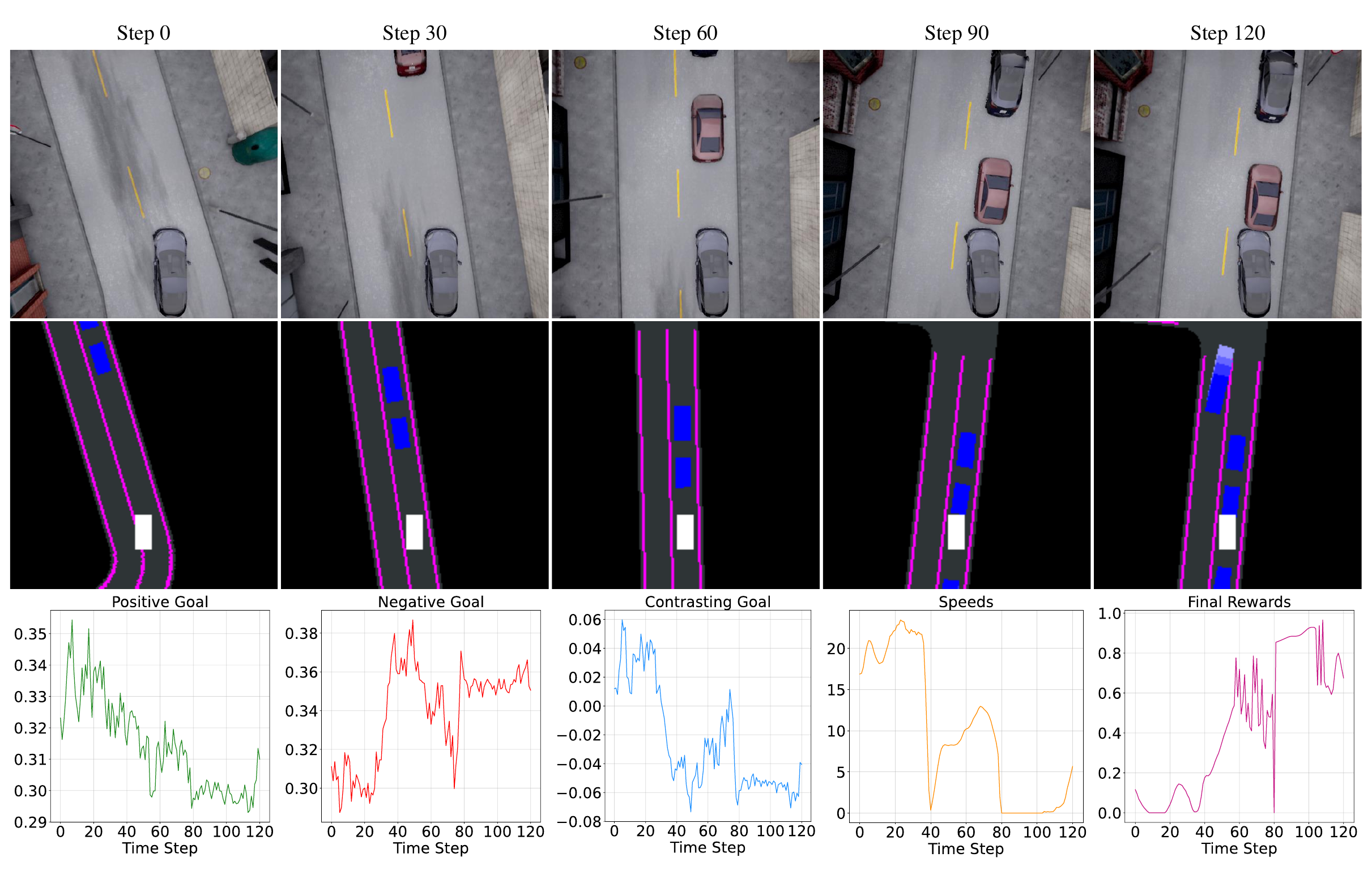}}
  \caption{CLIP rewards using different CLGs in different scenarios.}
  \label{case1}
\end{figure*}

Figs. \ref{case1} and \ref{case2} illustrate how our hierarchical reward synthesis approach combines VLM-derived semantic rewards with vehicle state information to generate comprehensive final rewards. The scenario in Fig. \ref{case1} shows a collision situation where the ego vehicle approaches a stationary vehicle ahead. The first row presents the RGB images, while the second row shows the corresponding semantic segmentation maps.
In the third row, we observe how different reward components evolve throughout the scenario. As the ego vehicle approaches the leading vehicle, the positive goal similarity steadily decreases while the negative goal similarity increases, resulting in a declining contrasting goal score. Similar to Fig. \ref{prompts_analysis} (a), this pattern effectively captures the increasing risk of collision. The speed profile reveals several distinct phases of the ego vehicle's behavior. Before step 40, the ego vehicle maintains a relatively high speed despite approaching the stationary vehicle, a behavior that our VLM identifies as potentially unsafe. Consequently, the final reward remains close to zero during this period, effectively penalizing this aggressive driving behavior.
Between steps 40 and 80, the ego vehicle significantly reduces its speed in response to the potential collision risk. This cautious behavior is rewarded with an increased final reward, demonstrating how our reward synthesis effectively encourages appropriate speed adjustments. During steps 80-100, when the ego vehicle comes to a complete stop behind the leading vehicle, the final reward reaches its highest values, validating that our reward function correctly identifies this as the desirable behavior in this scenario. However, after step 110, when the ego vehicle unexpectedly begins to move again despite the continued presence of the obstacle, we observe a sharp decline in the final reward, properly penalizing this undesired behavior.

\begin{figure*}
  \centerline{\includegraphics[width=0.993\textwidth]{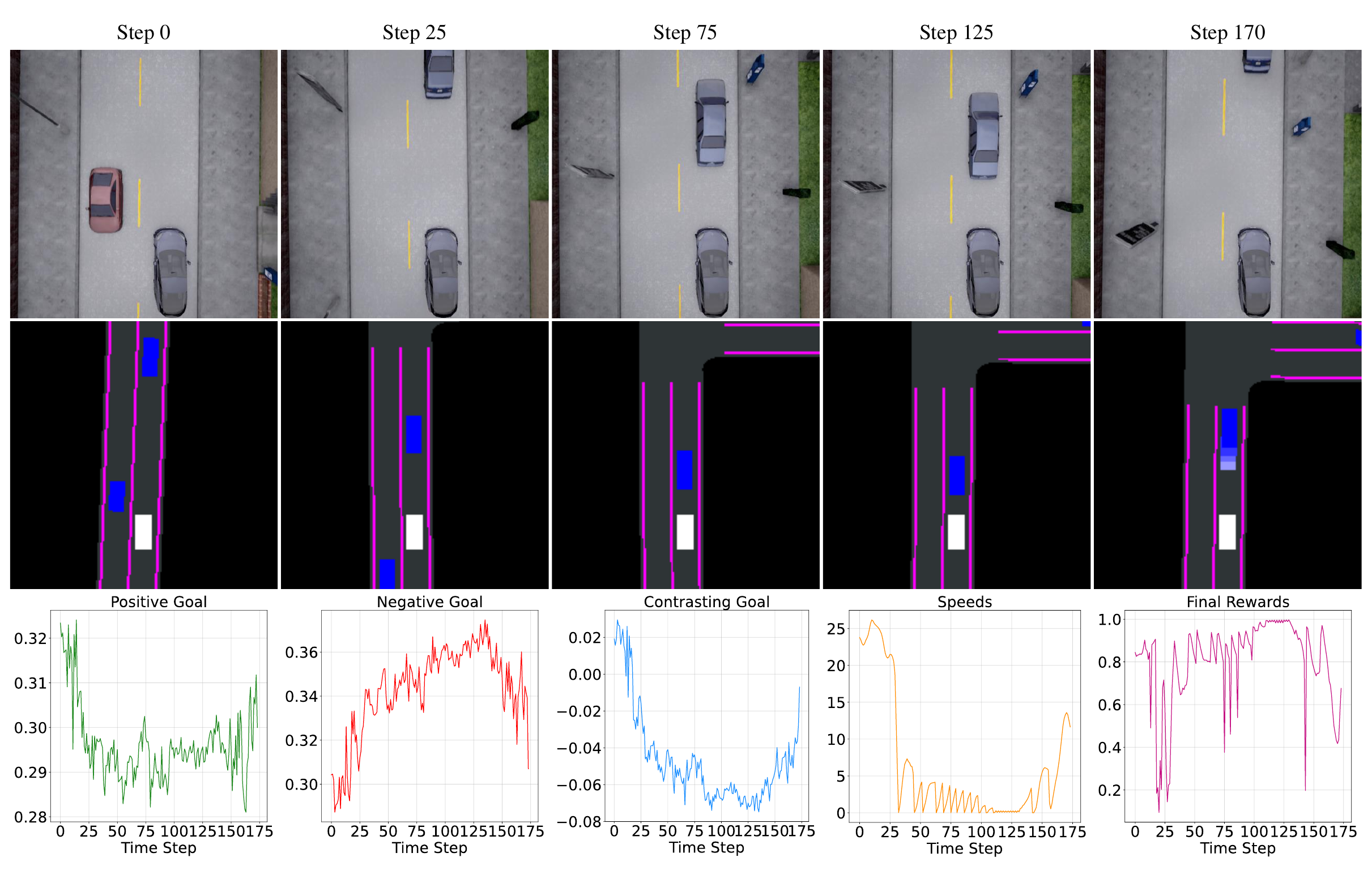}}
  \caption{CLIP rewards using different CLGs in different scenarios.}
  \label{case2}
\end{figure*}

Fig. \ref{case2} presents a successful driving scenario that demonstrates how our reward effectively guides the RL agent through complex, multi-phase driving situations. The scenario involves approaching a stationary vehicle, maintaining a safe following distance, and resuming motion when the leading vehicle departs.
The evolution of VLM-derived semantic rewards effectively captures the changing nature of the scene. The positive goal similarity initially decreases as the ego vehicle approaches the stationary vehicle, then increases after the leading vehicle departs and the road becomes clear. Conversely, the negative goal similarity rises during the approach phase and declines after the leading vehicle's departure. The contrasting goal synthesizes these patterns, providing a clear learning signal that aligns with safe driving behaviors.
The speed profile reveals several distinct driving phases that showcase the agent's learned behavior. Initially, the ego vehicle travels at a relatively high speed of approximately 25 km/h when the road is clear. Between steps 25 and 125, as it detects and responds to the stationary vehicle ahead, the speed gradually reduces to around 5 km/h before coming to a complete stop. The final reward during this period remains consistently high, with minor fluctuations around step 25 due to a slightly delayed braking response. After step 125, when the leading vehicle resumes motion, the ego vehicle demonstrates appropriate acceleration behavior, though we observe some reward fluctuations around step 160 due to momentarily aggressive acceleration that temporarily reduces the safety margin.

The cases in Figs. \ref{case1} and \ref{case2} study highlight how our hierarchical reward synthesis successfully guides the agent through different driving phases while maintaining safety. The final reward effectively balances the semantic understanding of the scene captured by the VLM-derived rewards with concrete vehicle states, encouraging behavior that is both efficient and safe. Our approach demonstrates robustness in handling dynamic situations, appropriately transitioning between different driving modes while maintaining high reward values for safe behavior and penalizing potentially risky actions through reward modulation.

\section{Conclusions and Future Work}
\label{Conclusions and Future Work}
This paper presented VLM-RL, a unified framework that integrates pre-trained Vision-Language Models (VLMs) with Reinforcement Learning (RL) for safe autonomous driving. The key contributions of this work lie in addressing the challenges of reward design by leveraging the semantic understanding capabilities of VLMs. Specifically, we introduced the Contrasting language goal (CLG)-as-reward paradigm, which utilizes both positive and negative language goals to generate semantic rewards. We further introduced a hierarchical reward synthesis method that combines CLG-based rewards with vehicle state information, improving reward stability and offering a more comprehensive learning signal. To mitigate computational challenges, a batch-processing technique was employed, enabling efficient reward computation during training. Extensive experiments conducted in the CARLA simulator demonstrated the efficacy of the VLM-RL. Compared to state-of-the-art baselines, VLM-RL achieved significant improvements in safety, efficiency, and generalization to unseen driving scenarios. Notably, VLM-RL reduced collision rates by 10.5\%, increased route completion rates by 104.6\%, and exhibited robust performance across diverse traffic environments. These results validate the potential of integrating VLMs with RL to develop more reliable and context-aware autonomous driving policies.

While VLM-RL shows promising results, several avenues for future work remain. First, improving the computational efficiency of VLM inference during training and deployment remains a challenge. Techniques such as model distillation or quantization could be investigated to reduce inference latency. Second, expanding the range of driving tasks and scenarios to include interactions with pedestrians, cyclists, and traffic signals—particularly traffic lights, which are not currently considered—could significantly enhance the framework's robustness and realism. Additionally, integrating human-in-the-loop feedback for refining CLG-based rewards dynamically may provide additional adaptability and personalization, further aligning driving policies with human preferences. In summary, VLM-RL provides a robust and scalable solution for reward design in RL-based autonomous driving, offering a promising direction for achieving safer and more generalizable driving policies. Continued research in this area has the potential to bridge the gap between simulation-based learning and real-world deployment, advancing the field of autonomous driving toward human-level safety and reliability.

\section*{Acknowledgment}
This work was supported by the University of Wisconsin-Madison's Center for Connected and Automated Transportation (CCAT), a part of the larger CCAT consortium, a USDOT Region 5 University Transportation Center funded by the U.S. Department of Transportation, Award \#69A3552348305. The contents of this paper reflect the views of the authors, who are responsible for the facts and the accuracy of the data presented herein, and do not necessarily reflect the official views or policies of the sponsoring organization.
	
\appendix

\section{Robustness of the CLG-as-Reward Paradigm}
\label{appendices1}
In this appendix, we show that the CLG reward defined in Eq.~(\ref{eq7}) also enhances robustness against uncertainties or adversarial perturbations. We first establish a lemma concerning the Lipschitz continuity of the cosine similarity measure. We then prove that $R_{\text{CLG}}$ inherits this Lipschitz continuity under suitable assumptions. Finally, we demonstrate how this continuity, combined with the presence of both positive and negative goals, improves the robustness of the learned policy.

\begin{lemma}[Lipschitz Continuity of Cosine Similarity]
\label{lemma:lipschitz_cosine}
Let $\mathbf{u}, \mathbf{v}, \mathbf{w} \in \mathbb{R}^k$ be unit vectors (i.e., $\|\mathbf{u}\| = \|\mathbf{v}\| = \|\mathbf{w}\| = 1$). Then, the cosine similarity function is Lipschitz continuous with Lipschitz constant $L = 1$. In particular, for any such vectors
\begin{equation}
\left|\text{sim}(\mathbf{u}, \mathbf{w}) - \text{sim}(\mathbf{v}, \mathbf{w})\right| \leq |\mathbf{u} - \mathbf{v}|
\end{equation}
\end{lemma}

\begin{proof}
Since $\mathbf{u}, \mathbf{v}, \mathbf{w}$ are unit vectors, the cosine similarity reduces to the inner product:
\begin{equation}
\text{sim}(\mathbf{u}, \mathbf{w}) = \mathbf{u}^\top \mathbf{w}
\end{equation}
Then,
\begin{align}
\left|\text{sim}(\mathbf{u}, \mathbf{w}) - \text{sim}(\mathbf{v}, \mathbf{w})\right|
&= |\mathbf{u}^\top \mathbf{w} - \mathbf{v}^\top \mathbf{w}| \\
&= |(\mathbf{u} - \mathbf{v})^\top \mathbf{w}| \\
&\leq |\mathbf{u} - \mathbf{v}| \cdot |\mathbf{w}| \quad  \text{(by Cauchy--Schwarz inequality)} \\
&= |\mathbf{u} - \mathbf{v}| \cdot 1 = |\mathbf{u} - \mathbf{v}|
\end{align}
\end{proof}

\begin{theorem}[Lipschitz Continuity of the CLG Reward Function] \label{theorem:lipschitz_reward} Suppose the VLM encoders produce unit-length embeddings. Under this assumption, $R_{\text{CLG}}$ is Lipschitz continuous with Lipschitz constant $L = 1$. Specifically, for any two states $s_1, s_2 \in \mathcal{S}$: 
\begin{equation} 
|R_{\text{CLG}}(s_1) - R_{\text{CLG}}(s_2)| \leq |\mathbf{v}_1 - \mathbf{v}_2|
\end{equation} 
where $\mathbf{v}_i = \text{VLM}_I(\psi(s_i))$.
\end{theorem}

\begin{proof} 
Recall the definition of $R_{\text{CLG}}$ from Eq.~(\ref{eq7}): 
\begin{equation} 
R_{\text{CLG}}(s) = \alpha \cdot \text{sim}(\mathbf{v}, \mathbf{v}_{\text{pos}}) - \beta \cdot \text{sim}(\mathbf{v}, \mathbf{v}_{\text{neg}}) 
\end{equation} 
with $\alpha, \beta > 0$ and $\alpha + \beta = 1$, and $\mathbf{v} = \text{VLM}_I(\psi(s))$.

For any $s_1, s_2$, let $\mathbf{v}_1 = \text{VLM}_I(\psi(s_1))$ and $\mathbf{v}_2 = \text{VLM}_I(\psi(s_2))$. Then: 
\begin{align}
|R_{\text{CLG}}(s_1) - R_{\text{CLG}}(s_2)|
&= \big|\alpha(\text{sim}(\mathbf{v}_1,\mathbf{v}_{\text{pos}}) - \text{sim}(\mathbf{v}_2,\mathbf{v}_{\text{pos}})) - \beta(\text{sim}(\mathbf{v}_1,\mathbf{v}_{\text{neg}}) - \text{sim}(\mathbf{v}_2,\mathbf{v}_{\text{neg}}))\big| \\
&\leq \alpha |\text{sim}(\mathbf{v}_1,\mathbf{v}_{\text{pos}}) - \text{sim}(\mathbf{v}_2,\mathbf{v}_{\text{pos}})| + \beta |\text{sim}(\mathbf{v}_1,\mathbf{v}_{\text{neg}}) - \text{sim}(\mathbf{v}_2,\mathbf{v}_{\text{neg}})|
\end{align}

By Lemma~\ref{lemma:lipschitz_cosine}, since embeddings are unit norm,
\begin{equation}
|\text{sim}(\mathbf{v}_1,\mathbf{v}_g) - \text{sim}(\mathbf{v}_2,\mathbf{v}_g)| \leq |\mathbf{v}_1 - \mathbf{v}_2|
\end{equation}
for $g \in \{\text{pos}, \text{neg}\}$. Substituting this into the inequality above:
\begin{align}
|R_{\text{CLG}}(s_1) - R_{\text{CLG}}(s_2)|
&\leq \alpha |\mathbf{v}_1 - \mathbf{v}_2| + \beta |\mathbf{v}_1 - \mathbf{v}_2| \\
&= (\alpha + \beta)|\mathbf{v}_1 - \mathbf{v}_2| = |\mathbf{v}_1 - \mathbf{v}_2|
\end{align}

Thus, $R_{\text{CLG}}$ is Lipschitz continuous with $L=1$.
\end{proof} 

\begin{remark} Lipschitz continuity ensures that small changes in the observation lead to small changes in the reward. This property helps stabilize RL training by reducing variance in gradient estimates and making the learning process more predictable \citep{bhandari2018finite}. \end{remark}

\begin{theorem}[Robustness Enhancement of the CLG Reward] \label{theorem:robustness_clg} 
Assume that both the VLM encoder and the observation function $\psi$ are Lipschitz continuous. Let $L_{\text{VLM}}$ and $L_{\psi}$ denote their respective Lipschitz constants. Consider a perturbed state $s' = s + \delta_s$, where $\delta_s$ is a small perturbation. Then there exists a constant $K = L_{\text{VLM}} L_{\psi}$ such that: \begin{equation} 
|R_{\text{CLG}}(s) - R_{\text{CLG}}(s')| \leq K |\delta_s| 
\end{equation}

This shows that the change in $R_{\text{CLG}}$ under perturbations is linearly bounded by the magnitude of the perturbation, making the policy more robust to adversarial or uncertain disturbances. 
\end{theorem}

\begin{proof} 
Define $\mathbf{v}_s = \text{VLM}_I(\psi(s))$ and $\mathbf{v}_{s'} = \text{VLM}_I(\psi(s'))$. If $\text{VLM}_I$ is $L_{\text{VLM}}$-Lipschitz and $\psi$ is $L_{\psi}$-Lipschitz, we have: 
\begin{equation} 
|\mathbf{v}_s - \mathbf{v}_{s'}| \leq L_{\text{VLM}}|\psi(s) - \psi(s')| \leq L_{\text{VLM}}L_{\psi}|\delta_s|
\end{equation}

From Thm~\ref{theorem:lipschitz_reward}, it follows that: \begin{equation} 
|R_{\text{CLG}}(s) - R_{\text{CLG}}(s')| \leq |\mathbf{v}_s - \mathbf{v}_{s'}|
\end{equation}

Combining these inequalities: 
\begin{equation} 
|R_{\text{CLG}}(s) - R_{\text{CLG}}(s')| \leq L_{\text{VLM}} L_{\psi}|\delta_s| = K|\delta_s|
\end{equation}

\end{proof}

This shows that the CLG reward fluctuation under perturbations is bounded. Now consider that in adversarial or uncertain environments, perturbations may selectively decrease similarity to the positive goal or increase similarity to the negative goal. Since $R_{\text{CLG}}$ includes both positive and negative goals, it provides a form of redundancy: if a perturbation maliciously affects one similarity, the other can partially compensate. A carefully chosen perturbation could drastically alter a single-target reward without this dual structure. Thus, the dual-goal structure of $R_{\text{CLG}}$ inherently enhances robustness. 

\begin{corollary} 
A policy trained with $R_{\text{CLG}}$ is more robust in adversarial or uncertain environments than one trained using only positive or only negative goals. The incorporation of both goals makes it harder for perturbations to significantly degrade the reward, thereby maintaining a stable and safe policy. 
\end{corollary}

\begin{proof}
The corollary follows directly from Theorem~\ref{theorem:robustness_clg} and the preceding analysis. If only a single goal (positive or negative) is used, a perturbation can be designed to specifically degrade that similarity score, causing a significant shift in the reward. However, $R_{\text{CLG}}$ combines both, making it harder for a perturbation to simultaneously degrade both terms advantageously, thus maintaining reward stability and robustness.
\end{proof}

\begin{remark}
These theoretical findings align with empirical results, where policies trained with $R_{\text{CLG}}$ show improved stability and resilience in complex, dynamic, or adversarial driving scenarios.
\end{remark}

\section{Workflow of Hierarchical Reward Synthesis}
\label{appendices2}
\begin{algorithm}[H]
\caption{Hierarchical Reward Synthesis}
\begin{algorithmic}[1]
\Require {Observation $o$, vehicle state $s$, positive language goal $l_{\text{pos}}$, negative language goal $l_{\text{neg}}$, weighting factors $\alpha$, $\beta$ (where $\alpha + \beta = 1$), and predefined thresholds $\theta_{\text{min}}, \theta_{\text{max}}$}
\State \textbf{Compute Semantic Reward:}
\State Compute image embedding: $\mathbf{v}_o \leftarrow \text{VLM}_I(o)$
\State Compute positive goal embedding: $\mathbf{v}_{\text{pos}} \leftarrow \text{VLM}_L(l_{\text{pos}})$
\State Compute negative goal embedding: $\mathbf{v}_{\text{neg}} \leftarrow \text{VLM}_L(l_{\text{neg}})$
\State Compute positive similarity using Eq.~(\ref{eq5}): $s_{\text{pos}} \leftarrow \text{sim}(\mathbf{v}_o, \mathbf{v}_{\text{pos}})$
\State Compute negative similarity using Eq.~(\ref{eq5}): $s_{\text{neg}} \leftarrow \text{sim}(\mathbf{v}_o, \mathbf{v}_{\text{neg}})$
\State Compute semantic reward using Eq.~(\ref{eq7}): $r_{\text{CLG}} \leftarrow \alpha \cdot s_{\text{pos}} - \beta \cdot s_{\text{neg}}$
\State Normalize $r_{\text{CLG}}$ to $[0, 1]$ using Eq.~(\ref{eq10}):
\State \hspace{1em} $r'_{\text{CLG}} \leftarrow 1 - \text{clip}(r_{\text{CLG}}, \theta_{\text{min}}, \theta_{\text{max}})$

\State \textbf{Compute Vehicle State Rewards:}
\State Compute target speed: $v_{\text{target}} \leftarrow r'_{\text{CLG}} \cdot v_{\text{max}}$
\State Compute speed alignment (as described in Section \ref{Framework: VLM-RL}): 
\State \hspace{1em} $r_{\text{speed}} \leftarrow 1 - \frac{|v - v_{\text{target}}|}{v_{\text{max}}}$
\State Compute $f_{\text{center}}(s)$ based on lateral deviation to lane center
\State Compute $f_{\text{angle}}(s)$ based on heading alignment
\State Compute $f_{\text{stability}}(s)$ based on lateral position stability

\State \textbf{Synthesize Reward:}
\State Compute final synthesis using Eq.~(\ref{eq11}):
\State \hspace{1em} $r_{\text{synthesis}} \leftarrow r_{\text{speed}} \times f_{\text{center}} \times f_{\text{angle}} \times f_{\text{stability}}$ \\

\Return $r_{\text{synthesis}}$
\end{algorithmic}
\end{algorithm}

\section{Convergence of the Synthesis Reward Function}
\label{appendices3}
\begin{theorem}[Convergence of Policy Optimization]
\label{theorem:convergence_final}
Under standard assumptions for the SAC algorithm \citep{haarnoja2018soft}, policy optimization with the synthesis reward $R_{\text{synthesis}}(s)$ defined in Eq.~(\ref{eq11}) converges to a local optimum.
\end{theorem}

\begin{proof}
The SAC algorithm seeks to maximize the expected discounted return, augmented by an entropy term:
\begin{equation}
J(\pi_\phi) = \mathbb{E}_{\pi_\phi} \left[ \sum_{t=0}^T \gamma^t (R_{\text{synthesis}}(s_t) + \alpha \mathcal{H}(\pi_\phi(\cdot|s_t))) \right],
\end{equation}
where $\alpha > 0$ is the temperature parameter controlling the trade-off between return and entropy, $\gamma \in [0,1)$ is the discount factor, and $\mathcal{H}(\pi_\phi(\cdot|s_t))$ denotes the entropy of the policy $\pi_\phi$ at state $s_t$.

By construction, the synthesis reward function $R_{\text{synthesis}}(s)$ in Eq.~(\ref{eq11}) is a product of several bounded and continuous components: \begin{itemize} 
\item $r'{_t^{\text{CLG}}}$ is bounded due to the clipping operation in its definition. 
\item Each vehicle state factor ($r_{\text{speed}}(s)$, $f_{\text{center}}(s)$, $f_{\text{angle}}(s)$, $f_{\text{stability}}(s)$) is designed to be bounded and continuous. In detail, $r_{\text{speed}}(s) \in [0,1]$, $f_{\text{center}}(s) \in [0,1]$, $f_{\text{angle}}(s) \in [0,1]$, $f_{\text{stability}}(s) \in [0,1]$. Thus, $R_{\text{synthesis}}(s)$ is itself bounded and continuous, and specifically $R_{\text{synthesis}}(s) \in [0,1]$.
\end{itemize}

Given $R_{\text{synthesis}}(s)$ is bounded and continuous, and assuming that the policy $\pi_\phi$ and the value function approximators are expressive enough (e.g., using neural networks with sufficient capacity), the SAC algorithm satisfies the convergence properties as established in \citep{haarnoja2018soft}.

Therefore, under these standard assumptions, the policy optimization using the synthesis reward function $R_{\text{synthesis}}(s)$ will converge to a local optimum. \end{proof}

\section{Stability of the Synthesis Reward Function}
\label{appendices4}
\begin{theorem}[Lipschitz Continuity of the Synthesis Reward Function] \label{theorem:lipschitz_final_reward} The synthesis reward function $R_{\text{synthesis}}(s)$ is Lipschitz continuous with respect to the state $s$. \end{theorem}

\begin{proof} To show Lipschitz continuity, we need to show that there exists a constant $L>0$ such that for any two states $s_1$ and $s_2$:

\begin{equation}
| R_{\text{synthesis}}(s_1) - R_{\text{synthesis}}(s_2) | \leq L | s_1 - s_2 |.
\end{equation}

The synthesis reward \( R_{\text{synthesis}}(s) \) is computed as:
\begin{equation}
R_{\text{synthesis}}(s) = r_{\text{speed}}(s) \cdot f_{\text{center}}(s) \cdot f_{\text{angle}}(s) \cdot f_{\text{stability}}(s),
\end{equation}
where \( r_{\text{speed}}(s) \), \( f_{\text{center}}(s) \), \( f_{\text{angle}}(s) \), and \(f_{\text{stability}}(s) \) are components that evaluate different aspects of the driving task. To prove Lipschitz continuity of \( R_{\text{synthesis}}(s) \), we analyze the individual components.

The speed alignment reward \( r_{\text{speed}}(s) = 1 - \frac{|v(s) - v_{\text{target}}(s)|}{v_{\text{max}}} \) is Lipschitz continuous because both the current speed \( v(s) \) and target speed \( v_{\text{target}}(s) = r'{_t^{\text{CLG}}} \cdot v_{\text{max}} \) are continuous functions of the state \( s \). Additionally, the absolute value function and normalization by \( v_{\text{max}} \) preserve Lipschitz continuity. Similarly, the lane centering factor \( f_{\text{center}}(s) \) and the heading alignment factor \( f_{\text{angle}}(s) \) are both continuous functions with respect to the state \( s \). The distance stability term \(f_{\text{stability}}(s) \) is also continuous due to its dependence on the positions of the ego vehicle.

Since each component of \( R_{\text{synthesis}}(s) \) is Lipschitz continuous and bounded within \([0, 1]\), their product is also Lipschitz continuous according to the properties of bounded Lipschitz functions. Thus, \( R_{\text{synthesis}}(s) \) satisfies the required Lipschitz condition.

\end{proof}

\begin{remark} Lipschitz continuity of the reward function ensures that small changes in the state lead to small changes in the reward. This property contributes to the stability of the learning process by preventing large fluctuations in the gradient estimates during policy optimization \citep{bhandari2018finite}. \end{remark}

\section{Workflow of Policy Training with Batch-Processing}
\label{appendices5}
\begin{algorithm}[H]
\caption{Policy Training with Batch-Processing and SAC}
\begin{algorithmic}[1]
\Require{Initial policy parameters $\phi$, Q-function parameters $\theta$, empty replay buffer $\mathcal{D}$, batch size $B$, total training steps $T$, VLM encoders $\text{VLM}_I$, $\text{VLM}_L$, language goals $l_{\text{pos}}$, $l_{\text{neg}}$, weighting factors $\alpha$, $\beta$ (where $\alpha + \beta = 1$), and predefined thresholds $\theta_{\text{min}}, \theta_{\text{max}}$}
\State \textbf{Precompute Goal Embeddings:}
\State $\mathbf{v}_{\text{pos}} \leftarrow \text{VLM}_L(l_{\text{pos}})$
\State $\mathbf{v}_{\text{neg}} \leftarrow \text{VLM}_L(l_{\text{neg}})$

\For{$t = 1$ to $T$}
    \State Observe state $s_t$ and raw observation $o_t$
    \State Select action $a_t \sim \pi_{\phi}(a_t | s_t)$
    \State Execute $a_t$ in environment
    \State Receive reward $r_t$ (initially set to $0$), next state $s_{t+1}$, and observation $o_{t+1}$
    \State Store transition $(o_t, s_t, a_t, r_t, o_{t+1}, s_{t+1})$ in $\mathcal{D}$

    \If{\text{time to update}}
        \State Sample mini-batch $\{(o_i, s_i, a_i, r_i, o_{i+1}, s_{i+1})\}_{i=1}^B$ from $\mathcal{D}$
        \For{each transition in mini-batch}
            \State \textbf{Compute Semantic Reward:}
            \State $\mathbf{v}_{o_i} \leftarrow \text{VLM}_I(o_i)$
            \State $s_{\text{pos}} \leftarrow \text{sim}(\mathbf{v}_{o_i}, \mathbf{v}_{\text{pos}})$
            \State $s_{\text{neg}} \leftarrow \text{sim}(\mathbf{v}_{o_i}, \mathbf{v}_{\text{neg}})$
            \State Compute CLG reward using Eq.~(\ref{eq7}): 
            \State \hspace{1em} $r_{\text{CLG}} \leftarrow \alpha \cdot s_{\text{pos}} - \beta \cdot s_{\text{neg}}$
            \State Normalize $r_{\text{CLG}}$ using Eq.~(\ref{eq10}):
            \State \hspace{1em} $r'{_t^{\text{CLG}}} = \frac{\text{clip}(r_t^{\text{CLG}}, \theta_{\text{min}}, \theta_{\text{max}}) - \theta_{\text{min}}}{\theta_{\text{max}} - \theta_{\text{min}}}$

            \State \textbf{Compute Vehicle State Rewards:}
            \State Compute $v_{\text{target}} \leftarrow r_{\text{CLG}} \cdot v_{\text{max}}$
            \State $r_{\text{speed}} \leftarrow 1 - \frac{|v_i - v_{\text{target}}|}{v_{\text{max}}}$
            \State Compute $f_{\text{center}}$, $f_{\text{angle}}$, $f_{\text{stability}}$ based on vehicle state

            \State \textbf{Compute Synthesis Reward:}
            \State Compute $r_{\text{synthesis}}$ using Eq.~(\ref{eq11}):
            \State \hspace{1em} $r_{\text{synthesis}} \leftarrow r_{\text{speed}} \times f_{\text{center}} \times f_{\text{angle}} \times f_{\text{stability}}$

            \State Update $r_i \leftarrow r_{\text{synthesis}}$ in replay buffer
        \EndFor

        \State \textbf{Update Critic Networks:}
        \State Update Q-function parameters $\theta$ by minimizing $J_Q(\theta)$ as in Eq.~(\ref{eq16})

        \State \textbf{Update Policy Network:}
        \State Update policy parameters $\phi$ by minimizing $J_{\pi}(\phi)$ as in Eq.~(\ref{eq14})

        \State \textbf{Adjust Temperature Parameter:}
    \EndIf
\EndFor
\end{algorithmic}
\end{algorithm}

\bibliography{mybibfile}
	
\end{document}

%% file: defs.tex
\def\x{{\bf x}}

\def\0{{\bf 0}}
\def\1{{\bf 1}}